\documentclass[journal]{IEEEtran}
\usepackage[dvips]{graphicx}
\usepackage{epsfig,multirow}
\usepackage{amsthm,amsmath,amssymb,mathrsfs,bm}
\usepackage{amsfonts,dsfont,color,bbm}
\usepackage{algorithm,algorithmic}
\usepackage[font=footnotesize,labelfont=bf]{caption}
\usepackage{subcaption}
\usepackage{cite,epstopdf,epsf,epsfig}
\usepackage{resizegather}
\usepackage{booktabs}
\usepackage{float}
\usepackage{mathtools}
\usepackage{xr}
\usepackage{dsfont}
\usepackage{enumitem}
\usepackage{tikz}
\usetikzlibrary{positioning, fit, arrows.meta, shapes,matrix,calc,decorations.pathreplacing}

\pgfkeys{tikz/mymatrixenv/.style={decoration=brace,every left delimiter/.style={xshift=3pt},every right delimiter/.style={xshift=-3pt}}}
\pgfkeys{tikz/mymatrix/.style={matrix of math nodes,left delimiter=[,right delimiter={]},inner sep=2pt,column sep=1em,row sep=0.5em,nodes={inner sep=0pt}}}
\pgfkeys{tikz/mymatrixbrace/.style={decorate,thick}}

\tikzset{
    cheating dash/.code args={on #1 off #2}{
        \csname tikz@addoption\endcsname{%
            \pgfgetpath\currentpath%
            \pgfprocessround{\currentpath}{\currentpath}%
            \csname pgf@decorate@parsesoftpath\endcsname{\currentpath}{\currentpath}%
            \pgfmathparse{\csname pgf@decorate@totalpathlength\endcsname-#1}\let\rest=\pgfmathresult%
            \pgfmathparse{#1+#2}\let\onoff=\pgfmathresult%
            \pgfmathparse{max(floor(\rest/\onoff), 1)}\let\nfullonoff=\pgfmathresult%
            \pgfmathparse{max((\rest-\onoff*\nfullonoff)/\nfullonoff+#2, #2)}\let\offexpand=\pgfmathresult%
            \pgfsetdash{{#1}{\offexpand}}{0pt}}%
    }
}

\theoremstyle{break}

\newtheorem{theorem}{Theorem}
\newtheorem{remark}{Remark}
\newtheorem{lemma}{Lemma}
\newtheorem{prop}{Proposition}
\theoremstyle{definition}

\usepackage{balance}

\DeclareMathOperator*{\argmin}{arg\,min}

\newcommand{\Htn}{\textbf{H}_t}
\newcommand{\htn}{\textbf{h}_t}
\newcommand{\htt}{\textbf{h}_\tau}
\newcommand{\httm}{\textbf{h}_{\tau-1}}

\newcommand{\htm}{\textbf{h}_{t-1}}

\newcommand{\htp}{\textbf{h}_{t+1}}
\newcommand{\htb}{\textbf{h}_{1}}

\newcommand{\Xtp}{\textbf{X}_{t+1}}
\newcommand{\Xt}{\textbf{X}_t}

\newcommand{\xtp}{\textbf{x}_{t+1}}
\newcommand{\xt}{\textbf{x}_t}

\newcommand{\I}{\textbf{I}}
\newcommand{\C}{\bm{\vartheta}}
\newcommand{\W}{\textbf{W}}
\newcommand{\U}{\textbf{U}}

\newcommand{\del}{\partial}

\newcommand{\thk}{\bm{\theta}}
\newcommand{\muk}{\bm{\mu}}
\newcommand{\vect}{\textrm{vec}}

\newcommand{\loss}{\ell_t(\thk,\muk)}
\newcommand{\losst}{\ell_t(\thk',\muk)}

\newcommand{\phtnt}{\frac{\del \htn}{\del \htt}}
\newcommand{\phtntp}{\frac{\del \htn'}{\del \htt'}}

\newcommand{\phtth}{\frac{\del \htt}{\del \thk}}
\newcommand{\phtthp}{\frac{\del \htt'}{\del \thk}}

\newcommand{\phtmu}{\frac{\del \htt}{\del \muk}}

\newcommand{\phtmup}{\frac{\del \htt'}{\del \muk}}

\newcommand{\dt}{d_t}
\newcommand{\dht}{\hat{d}_t}

\newcommand{\pt}{p_t}
\newcommand{\pht}{\hat{p}_t}

\DeclarePairedDelimiter{\norm}{\lVert}{\rVert}
\DeclarePairedDelimiter{\norminf}{\lVert}{\rVert_\infty}
\DeclarePairedDelimiter{\bOh}{O\big(}{\big)}

\DeclarePairedDelimiter{\bnorm}{\Big\lVert}{\Big\rVert}%

\DeclarePairedDelimiter{\anglep}{\langle}{\rangle}
\DeclarePairedDelimiter{\banglep}{\Big\langle}{\Big\rangle}

\newcommand{\Wt}{\W_t}
\newcommand{\Ut}{\U_t}
\newcommand{\Ct}{\C_t}

\newcommand{\Kt}{\mathcal{K}_\theta}
\newcommand{\Km}{\mathcal{K}_\mu}
\newcommand{\Kcc}{\mathcal{K}_\vartheta}

\newcommand{\nx}{n_x}
\newcommand{\nh}{n_h}

\newcommand{\lossgen}{\ell_t(\thk,\muk)}

\newcommand{\slossgen}{L_{t,w}(\thk,\muk)}
\newcommand{\slossgent}{L_{t,w}(\thk_t,\muk_t)}
\newcommand{\slossgentp}{L_{t,w}(\thk_{t+1},\muk_{t+1})}

\newcommand{\projTh}{\frac{\del_{\Kt} \lossgen }{\del \thk}}

\newcommand{\projLTht}{\frac{\del_{\Kt}  \slossgent }{\del \thk}}
\newcommand{\partTh}{\frac{\del \lossgen }{\del \thk}}
\newcommand{\partLTht}{\frac{\del \slossgent }{\del \thk}}
\newcommand{\partLThtp}{\frac{\del \slossgentp }{\del \thk}}

\newcommand{\shortLTht}{\frac{\del L_{t,w} }{\del \thk}}
\newcommand{\shortLprojTht}{\frac{\del_{\Kt}  L_{t,w} }{\del \thk}}

\newcommand{\projM}{\frac{\del_{\Km} \lossgen }{\del \muk}}

\newcommand{\projLMt}{\frac{\del_{\Km}  \slossgent }{\del \muk}}
\newcommand{\partM}{\frac{\del \lossgen }{\del \muk}}
\newcommand{\partLMt}{\frac{\del \slossgent }{\del \muk}}
\newcommand{\partLMtp}{\frac{\del \slossgentp }{\del \muk}}

\newcommand{\shortLMt}{\frac{\del L_{t,w} }{\del \muk}}
\newcommand{\shortLprojMt}{\frac{\del_{\Km}  L_{t,w} }{\del \muk}}

\newcommand{\bTh}{\beta_\theta}
\newcommand{\bMu}{\beta_\mu}
\newcommand{\bThMu}{\beta_{\theta\mu}}

\begin{document}
\title{Achieving Online Regression Performance of LSTMs with Simple RNNs}

\author{N. Mert Vural, Fatih Ilhan, Selim F. Yilmaz, Salih Erg\"{u}t and Suleyman Serdar Kozat, \textit{Senior Member, IEEE}
\thanks{This work is supported in part by TUBITAK Contract No. 117E153.}
\thanks{ N. M. Vural, F. Ilhan, S. F. Yilmaz and S. S. Kozat are with the Department of Electrical and Electronics Engineering, Bilkent University, Ankara 06800, Turkey, e-mail: \{vural, filhan, syilmaz, kozat\}@ee.bilkent.edu.tr.}
\thanks{S. Erg\"{u}t is with the Turkcell Technology, 5G R\&D Team, Istanbul, Turkey, e-mail: salih.ergut@turkcell.com.tr}
\thanks{S. S. Kozat and F. Ilhan are also with DataBoss A.S., Bilkent Cyberpark, Ankara, 06800, emails:\{serdar.kozat, fatih.ilhan\}@data-boss.com.tr. }
}

\maketitle
\begin{abstract}
Recurrent Neural Networks (RNNs) are widely used for online regression due to their ability to generalize nonlinear temporal dependencies. As an RNN model, Long-Short-Term-Memory Networks (LSTMs) are commonly preferred in practice, as these networks are capable of learning long-term dependencies while avoiding the vanishing gradient problem. However, due to their large number of parameters, training LSTMs requires considerably longer training time compared to simple RNNs (SRNNs). In this paper, we achieve the online regression performance of LSTMs with SRNNs efficiently. To this end, we introduce a first-order training algorithm with a linear time complexity in the number of parameters. We show that when SRNNs are trained with our algorithm, they provide very similar regression performance with the LSTMs in two to three times shorter training time. We provide strong theoretical analysis to support our experimental results  by providing regret bounds on the convergence rate of our algorithm. Through an extensive set of experiments, we verify our theoretical work and demonstrate significant performance improvements of our algorithm with respect to LSTMs and the other state-of-the-art learning models.
\end{abstract}

\section{Introduction}
\subsection{Preliminaries}
Estimating an unknown desired signal is one of the main subjects of interest in the contemporary online learning literature\cite{CB2006}. In this problem, a learner sequentially receives a data sequence related to a desired signal to predict the signal's next value. This problem (also known as online regression) is extensively studied due to its applications in a wide set of problems, e.g., neural network training\cite{Cheng11,Jinz15}, signal processing\cite{LaxhammarF14}, and machine learning\cite{LSTMSearch}. In these studies, nonlinear approaches are commonly employed since linear modeling is inadequate for a wide range of applications due to the constraints of linearity. 


For online regression, there exists a wide range of established nonlinear approaches in the machine learning literature\cite{Kivinen04,tree}. However, these approaches usually suffer from prohibitive computational requirements and may provide poor performance due to overfitting and stability issues\cite{overfit}. Adopting neural networks is another method due to their high generalization capacity\cite{Specht91}. However, neural-network-based algorithms are shown to provide inadequate performance in certain applications\cite{Zhao19}. To overcome these limitations, neural networks composed of multiple layers, i.e., deep neural networks (DNNs), have been introduced. In the last few years, DNNs have led to outstanding performance on a variety of problems, such as visual recognition, speech recognition, and natural language processing\cite{lecun2015}. Consequently, they have become a widely accepted tool for applications that require nonlinear data processing. 

On the other hand,  DNNs lack temporal memory. Therefore, they provide only limited performance in processing temporal data and time series, which are commonly encountered in online regression problems\cite{Boden05}. To remedy this issue,  recurrent neural networks (RNNs) are used, as these networks  store history in their state representation. However, simple RNNs (SRNNs) are shown to incur either exponential growth or decay in the norm of gradients, which are the well-known exploding and vanishing gradient problems, respectively\cite{Bengio94}. Therefore, they are insufficient to capture long-term dependencies, which significantly restricts their performance in real-life applications. To resolve this issue, a novel RNN architecture with several control structures, i.e., long short-term-memory network (LSTM), was introduced~\cite{LSTM}. Through many variants, LSTMs have seen remarkable empirical success in a broad range of sequential learning tasks, including online regression~\cite{Gers2002,GRU,Tolga2018}.

LSTMs maintain constant backward flow in the error signal to overcome the vanishing gradient problem. They utilize three gates, i.e., input, forget, and output gates, to regulate the information flow through the next steps. In each gate, an independent set of weights, i.e., gate weights, are employed to learn the optimal information regulation in a data-dependent manner. With this structure, LSTMs enhance the regression performance of SRNNs dramatically.  However, LSTMs require a comparably large number of parameters to provide this improvement, which makes them considerably more demanding in terms of computational and data requirements compared to SRNNs\cite{Tolga2018}. 

In this study, we obtain the regression performance of LSTMs with significantly less training time by using SRNNs.  In particular, we introduce an efficient first-order training algorithm, which requires only a linear time complexity in the number of parameters, and show that when our algorithm is applied to SRNNs, it achieves the online regression performance of the LSTMs. Since SRNNs have much fewer parameters compared to LSTMs, our algorithm provides that equivalent performance \emph{in two to three times shorter training time}. We note that to the best of our knowledge, as the first time in the literature, we obtain the online regression performance of LSTMs with SRNNs by using a first-order training algorithm. In the following, we justify our results by providing theoretical derivations and strong regret bounds on the convergence rate of our algorithm. Moreover, through an extensive set of experiments on real and synthetic data, we verify our theoretical work and demonstrate the performance improvements of our algorithm with respect to LSTMs and the state-of-the-art learning models.

\subsection{Prior Art and Comparison}

RNNs are a class of models with an internal memory due to recurrent feed-back connections, which makes them suitable for dealing with sequential problems, such as online prediction or regression\cite{Elman90}. However, SRNNs trained with Stochastic Gradient Descent (SGD) have difficulty in learning long-term dependencies encoded in the input sequences due to the vanishing gradient problem\cite{Bengio94}. The problem has been addressed by using a specialized architecture, namely LSTM, which maintains constant backward flow in the error signal controlled by several gates to regulate the information flow inside the network\cite{LSTM}. However,  the cost of utilizing additional gating structures is high computational complexity and larger training data to effectively learn significantly more parameters than SRNNs\cite{Tolga2018}. Although there exist models such as Gated Recurrent Units (GRUs), which use weight sharing to reduce the total number of parameters, there is still three to five times difference between the parameter size of GRUs/LSTMs and SRNNs for a fixed hidden-layer size\cite{LSTM,GRU,Gers2002}. We point out that the difference in the number of parameters increases the training time of GRUs/LSTMs with the same ratio (if not more) compared to SRNNs.

There exists a wide range of work for obtaining the performance of the LSTMs with simpler models\cite{SutskeverH10,narxrnn}. One of the most common approaches to that end is allowing SRNNs to have direct connections from the distant past to enable them to handle long-term dependencies.  One model in this line is the NARX-RNN, which introduces an additional set of recurrent connections with time lags of $2,3, \cdots , k$ time steps\cite{narxrnn}. However, NARX-RNNs are extremely inefficient, as both parameter counts and computation counts grow by the same factor $k$\cite{SutskeverH10}. A more efficient model with a similar structure is Clockwork RNNs, i.e., CW-RNNs\cite{cwrnn}. The main idea of CW-RNNs is splitting the weights and hidden units of SRNNs into partitions, each with a distinct period, to maintain an additive error flow. However, CW-RNNs require hidden units to be partitioned a priori, which in practice is difficult to do in any meaningful way, especially in online regression where we usually have a very small amount of a priori information on the data sequence. 

Another approach to improve the performance of SRNNs is to use higher-order properties of the error surface during training. For example, RMSprop utilizes approximate Hessian-based preconditioning, which significantly improves the error performance of SGD\cite{Dauphin15}. Due to its efficiency and performance improvement, RMSprop and its follow-on methods, such as Adam\cite{King14}, are being frequently used in the current deep learning applications. However, these techniques are not sufficient to close the gap between SRNNs and LSTMs in general. As a second-order technique, Kalman filters have shown advantages in bridging long time lags as well\cite{Perez02}. However, this approach requires a quadratic time complexity in the parameter size, which is computationally unfeasible for larger networks. Another attempt to bridge the performance of LSTMs and SRNNs is Hessian Free optimization, an adapted second-order training method that has been demonstrated to work well with SRNNs\cite{Martens2011}. However, the Hessian Free algorithm requires a large size of batches to approximate the Hessian matrix, which makes it impractical for the online settings.

In this paper, we obtain the online regression performance of LSTMs with SRNNs by using a first-order training algorithm. Our study differs from the model-based studies since we do not use any complex model rather than the standard Elman Network. Moreover, our algorithm is fundamentally different from the Kalman Filter and Hessian Free algorithms since it is a truly online first-order training algorithm. In the experiments, we show that our algorithm provides very similar performance with LSTMs trained with Adam, RMSprop and SGD.  Therefore, we introduce a highly practical algorithm, which provides comparable performance with the state-of-the-art methods in a highly efficient manner.

\subsection{Contributions}
Our contributions can be summarized as follows:
\begin{itemize}
\item To the best of our knowledge, we, as the first time in the literature, obtain the online regression performance of LSTMs with SRNNs by using a first-order training algorithm.
\item Our algorithm requires only a linear time complexity in the number of parameters, i.e., the same time complexity as SGD\cite{Will95}.
\item Since SRNNs have four times fewer parameters compared to LSTMs for a fixed hidden layer size, our algorithm obtains the performance of LSTMs in two to three times shorter training time.
\item We support our experimental results with strong theoretical analysis and prove that our algorithm converges to the local optimum parameters. In our proofs, we avoid using statistical assumptions, which makes our algorithm applicable to a wide range of adaptive signal processing applications, such as time-series prediction~\cite{lstmdekf} and object tracking~\cite{Coskun17}.
\item Through an extensive set of experiments on real and synthetic data, we verify our theoretical work and demonstrate significant performance improvements of our algorithm with respect to LSTMs and the state-of-the-art training methods.
\end{itemize}

\subsection{Organization of the Paper}
This paper is organized as follows: In Section \ref{sec:model}, we formally introduce the online regression problem and describe our RNN model. In Section \ref{sec:main}, we  develop  a first-order optimization algorithm with a theoretical convergence guarantee. In Section \ref{sec:sim}, we verify our results and demonstrate
the performance of our algorithm with numerical simulations. In Section \ref{sec:concl}, we conclude our paper with final remarks. 

\section{Model and Problem Definition} 
\label{sec:model}
All vectors are column vectors and denoted by boldface lower case letters. Matrices are represented by boldface capital letters. We use  $\norm{\cdot}$ to denote the $\ell_2$ vector or matrix norms depending on the argument. We use bracket notation $[n]$ to denote the set of the first $n$ positive integers, i.e., $[n]= \{ 1, \cdots, n \}$.

We study online regression. In this problem, we observe input variables $\{\textbf{x}_t\}_{t \geq 1}$ at each time step in a sequential manner. After receiving the vector $\xt$, we produce an estimate $\dht$ for the next unobserved target value $\dt$.\footnote{For mathematical convenience, we assume $\textbf{x}_t \in [-1,1]^{\nx}$ and $\dt \in  [-\sqrt{\nh},\sqrt{\nh}]$, where $\nh \in \mathbbm{N}$ is the hidden-size of the network. However, our derivations can be extended  to any bounded input and output sequences after shifting and scaling the magnitude.} After declaring our estimate, we observe $\dt$ and suffer the loss $\ell(\dt,\dht)$. Our aim is to optimize the network with respect to the loss function $\ell(\cdot,\cdot)$ in an online manner. In this study, we particularly work with the squared loss, i.e., $\ell(\dht,\dt)= 0.5 (\dt - \dht)^2$, where we scale the squared loss with $0.5$ for mathematical convenience when computing the derivatives. However, our work can be directly extended to the  cross-entropy loss, which is detailed in Appendix A. 

In this paper, we study online regression with nonlinear state-space models. As the parametric nonlinear state-space model, we use the SRNNs (or the standard Elman network model), i.e.,
\begin{align}
&\htn = \tanh(\W \htm + \U \xt ) \label{model1} \\
&\dht =  \C^T \htn.\label{model2}
\end{align}
Here, we have $\W \in \mathbbm{R}^{\nh \times \nh}$ and $\U \in \mathbbm{R}^{\nh \times \nx}$ as the hidden layer weight matrices. We have $\C \in \mathbbm{R}^{\nh}$, with $\norm{\C} \leq 1$, as the output layer weights.\footnote{Note that for the convergence of the training, the output-layer weights, i.e., $\C$, should stay bounded. To this end, our algorithm performs projection onto a convex set to keep $\C$ bounded. As the convex set, we use  $\{ \C: \norm{\C} \leq 1\}$ for mathematical convenience in our proofs.  However, our derivations can be extended to any  bounded set of $\C$ by shifting and scaling the magnitude.} Moreover, $\htn \in [-1,1]^{\nh}$ is the hidden state vector, $\xt \in [-1,1]^{\nx}$ is the input vector, $\dht \in [-\sqrt{\nh},\sqrt{\nh}]$ is our estimation, and $\tanh$ applies to vectors point-wise.  We note that  although we do not explicitly write the bias terms, they can be included in (\ref{model1})-(\ref{model2}) by augmenting the input vectors with a constant dimension. 

\section{Algorithm Development}
\label{sec:main}
In this section, we introduce the main contribution of our paper, i.e., a first-order algorithm that achieves the online regression performance of LSTMs with SRNNs. We develop our algorithm in three subsections. In the first subsection, we describe our approach and provide definitions for the following analysis. In the second subsection, we present the auxiliary results that will be used in the development of the main algorithm. In the last subsection, we introduce our algorithm and provide a regret bound on its convergence rate.

\subsection{Online Learning Approach}
\begin{figure*}[t!]
\hspace{10mm}
\begin{tikzpicture}[
    font=\sf \scriptsize,
    >=LaTeX,
    cell/.style={
        rectangle, 
        rounded corners=1mm, 
        draw,
        very thick,
        },
    invcell/.style={
        rectangle, 
        rounded corners=1mm, 
        draw=none,
        },    
    operator/.style={
        circle,
        draw,
        inner sep=-0.5pt,
        minimum height =.2cm,
        },
    function/.style={
        ellipse,
        draw,
        inner sep=1pt
        },
    ct/.style={
        },
    gt/.style={
        rectangle,
        draw,
        thick,
        minimum width=4mm,
        minimum height=3mm,
        inner sep=1pt
        },
    mylabel/.style={
        font=\scriptsize\sffamily
        },
    ArrowC1/.style={
        rounded corners=.25cm,
        thick,
        },
    ArrowC2/.style={
        rounded corners=.5cm,
        thick,
        },
    ]
	
    \node [cell, minimum height =1.75cm, minimum width=2.5cm, align=center] (cell1) at (-7.25,0){SRNN pass -defined in (\ref{model1})-\\   parametrized by \\ $\thk_t$ and $\muk_t$} ;
    \node [cell, minimum height =1.75cm, minimum width=2.5cm, align=center] (cell2) at (-3.75,0){SRNN pass \\   parametrized by \\ $\thk_t$ and $\muk_t$} ;
    \node [invcell, minimum height =1.75cm, minimum width=2.5cm, align=center] (cell3) at (-0.5,0){\Large{$\cdots$}} ;
    \node [cell, minimum height =1.75cm, minimum width=2.5cm, align=center] (cell4) at (2.75,0){SRNN pass \\   parametrized by \\ $\thk_t$ and $\muk_t$} ;
    
    \node[ct] (x1) at (-7.25,-1.5) {$\textbf{x}_{1}$};
    \node[ct] (x2) at (-3.75,-1.5) {$\textbf{x}_{2}$};
    \node[ct] (xt) at (2.75,-1.5) {$\textbf{x}_{t}$};
    
    \node[ct] (h0) at (-9.5,0) {$\textbf{h}_{0}$};

    \node [ct, minimum width=1 cm] (out) at (4.375,1.3) {\hspace{3.5mm} $\Ct^T \htn(\thk_t,\muk_t) = \dht$};

	\draw[decoration={brace,mirror,raise=5pt},decorate,thick] (-7.4, -1.2) -- node[left=5pt,align=center] {Input \\ Vectors} (-7.4,- 1.7);  	
 	\draw[decoration={brace,mirror,raise=5pt},decorate,thick]  (-9.55,0.3) -- node[left=5.05pt,align=center] {Initial \\ State} (-9.55,-0.3); 
	\draw[decoration={brace,mirror,raise=5pt},decorate,thick] (-9,-1.65) -- node[below=8.5pt,align=center] {Unfolded version of the SRNN model in (\ref{model1})-(\ref{model2}) over all the time steps up to the current time step $t$. \\ Note that all forward passes share the same parameters, i.e., $\thk_t$ and $\muk_t$.} (4.05,-1.65); 
    \draw[decoration={brace,raise=3pt},decorate,thick] (3.35,1.5) --  node[above=5.1pt,align=center] {Output Layer } (5.25,1.5); 

	\draw [->, thick] (x1.north) -| (cell1.south);
	\draw [->, thick] (x2.north) -| (cell2.south);
	\draw [->, thick] (xt.north) -| (cell4.south);
 
 	\draw [->, thick] (h0.east) -- (h0-|cell1.west);

	\draw [->, thick] (h0-|cell1.east) -- node[above] {$\textbf{h}_1$} (h0-|cell2.west);
	
	\draw [->, thick] (h0-|cell2.east) -- node[above] {$\textbf{h}_2$} (h0-|cell3.west);
	
	\draw [->, thick] (h0-|cell3.east) -- node[above] {$\textbf{h}_{t-1}$} (h0-|cell4.west);

	\draw [->, thick] (h0-|cell4.east) -| node[right] {\normalsize{$\htn(\thk_t,\muk_t)$}} (out.south);

\end{tikzpicture}  \\
 \caption{In this figure, we visually describe $\htn(\thk_t,\muk_t)$. $\htn(\thk_t,\muk_t)$ is defined as the hidden state vector obtained by running the model in (\ref{model1})-(\ref{model2}) with $\thk_t$ and $\muk_t$ from the initial time step up to the current time step $t$. In the figure, the SRNN sequence is initialized with a predetermined initial state $\textbf{h}_0$, which is independent of the network weights. Then, the same SRNN forward pass (given in (\ref{model1})) is repeatedly applied to the input sequence $\{\textbf{x}_t\}_{t\geq1}$, where all the iterations are parametrized by $\thk_t$ and $\muk_t$. The resulting hidden vector after $t$ iterations is defined to be $\htn(\thk_t,\muk_t)$. Here, we note that the dependence of $\htn(\cdot,\cdot)$ on $t$ is due to the increased length of the recursion at each time step.}
 \label{fig:rnn_it}
\end{figure*}
We investigate the SRNN-based regression problem in the online learning (or the online optimization) framework \cite{ZinkOGD}. In this framework, data becomes available sequentially, and the learning procedure continues through updating the model as new data instances are accessed. In order model chaotic, non-stationary, or even adversarial environments, no statistical assumptions are made on data. Therefore, the learning procedure is modeled as a game between a learner (or a learning algorithm) and a possibly adversarial environment, where the learner is tasked to predict model parameters from a convex set. In this setting, 
\begin{itemize}
\item First, the learner makes its prediction about model parameters.
\item Then, the (possibly adversarial) environment chooses a loss function.
\item Finally, the learner observes its loss and takes an action (i.e., update model parameters) to minimize its cumulative loss.
\end{itemize}
This procedure is repeated at each round $t$ until all data instances are utilized.

To formulate the SRNN-based regression problem in the online learning setting, let us introduce three new notations: We use $\W_t$, $\U_t$ and $\Ct$ to denote the weights learned by the first $t-1$ data/input pairs. For mathematical convenience, we work with the vectorized form of the weight matrices, i.e., $\thk_t =\vect(\Wt)$ and $\muk_t=\vect(\Ut)$. Moreover, we use $\htn(\thk_t,\muk_t)$ to denote the hidden state vector obtained by running the SRNN model in (\ref{model1})-(\ref{model2}) with $\thk_t$ and $\muk_t$ from the initial time step up to the current time step $t$ (for the detailed description of $\htn(\thk_t,\muk_t)$ see Fig. \ref{fig:rnn_it}).

Then, we construct our learning setting as follows: At each round $t$, the learner declares his prediction $\thk_t$ and $\muk_t$; concurrently, the adversary chooses a target value $\dt \in [-\sqrt{\nh},\sqrt{\nh}]$, an input $\xt \in [-1,1]^{n_x}$, and a weight vector $\Ct$, $\norm{\Ct} \leq 1$; then, the learner observes the loss function
\begin{equation}
\ell_t(\thk_t,\muk_t) \coloneqq 0.5 \big( \dt - \underbrace{\Ct^T \htn(\thk_t,\muk_t)}_{\dht} \big)^2 \label{def:loss}
\end{equation} 
and suffers the loss $\ell_t(\thk_t,\muk_t)$. This procedure of play is repeated across $T$ rounds, where $T$ is the total number of input instances. We note that we constructed our setting for adversarial $\Ct$ selections for mathematical convenience in our proofs. However, since the selected $\ell_t(\thk_t,\muk_t)$ is convex with respect to $\Ct$, we will use the online gradient descent algorithm \cite{ZinkOGD} to learn the optimal output layer weights (simultaneously with the hidden layer weights) during the training.

Since in the online learning framework, no statistical assumptions are made on the input/output sequences, the performance of the algorithms is analyzed  with the notion of \emph{regret}. However, the standard regret definition for the convex problems is intractable in the non-convex settings due to the NP-hardness of the non-convex global optimization \cite{HazanNonConvex}. Therefore, we use the notion of \emph{local regret} recently introduced by Hazan et al.\cite{HazanNonConvex}, which quantifies the objective of predicting points with a small gradient on average.

To formulate the local regret for our setting, we first define the projected partial derivatives of $\lossgen$ with respect to $\thk$ and $\muk$ as follows:
\begin{align}
&\projTh \coloneqq \frac{1}{\eta} \Big( \thk - \Pi_{\Kt} \! \Big[\thk - \eta \partTh \Big] \Big) \label{def:projTh} \\
&\projM \! \coloneqq \frac{1}{\eta} \Big( \muk - \Pi_{\Km} \! \Big[\muk - \eta \partM \Big] \Big),\label{def:projM}
\end{align} 
where $\del_{\mathcal{K}}$ denotes the projected partial derivative operator defined with some convex set $\mathcal{K}$ and some learning rate $\eta$. The operators $\Pi_{\Kt}\![\cdot]$ and $\Pi_{\Km}\![\cdot]$ denote the orthogonal projections onto $\Kt$ and $\Km$.

We define the time-smoothed loss at time $t$, parametrized by some window-size $w \in [T]$, as
\begin{equation}
\slossgen \coloneqq \frac{1}{w} \sum_{i=0}^{w-1} \ell_{t-i} (\thk,\muk). \label{def:tsloss}
\end{equation}
Then, we define the local regret as
\begin{equation}
R_w(T) \! \coloneqq \! \sum_{t=1}^T \! \Big( \bnorm{\projLTht}^2 \! + \! \bnorm{\projLMt}^2 \Big). \label{def:locreg}
\end{equation}
Our aim is to derive a sublinear upper bound for $R_w(T)$ in order to ensure the convergence of our algorithm to the locally optimum weights. However, before the derivations, we first present the auxiliary results, i.e., the Lipschitz and smoothness properties of $L_{t,w}(\thk,\muk)$, which will be used in the convergence proof of our algorithm.

\subsection{Lipschitz and Smoothness Properties}
In this section, we derive the Lipschitz and smoothness properties of the time-smoothed loss function $L_{t,w}(\thk,\muk)$. 

We note that $L_{t,w}(\thk,\muk)$ is defined as the average of the most recent $w$ instant loss functions (see (\ref{def:tsloss})), where the loss function $\loss$ recursively depends on $\thk$ and $\muk$ due to $\htn(\thk,\muk)$ (see (\ref{def:loss}) and Fig. \ref{fig:rnn_it}). We emphasize that since we are interested in online learning, this recursion can be infinitely long, which might cause  $L_{t,w}(\thk,\muk)$  to have \emph{unboundedly large derivatives}. On the other hand, online algorithms naturally require \emph{loss functions with bounded gradients} to guarantee convergence. 

Therefore, in this part, we analyze the recursive dependency of  $L_{t,w}(\thk,\muk)$ on $\muk$ and $\thk$. We derive sufficient conditions for its derivatives to be bounded and find the explicit formulations of the smoothness constants of $L_{t,w}(\thk,\muk)$ in terms of the model parameters. To this end, we first derive the Lipschitz properties of $\htn(\thk,\muk)$, and observe the effect of infinitely long recursion on the derivatives of $L_{t,w}(\thk,\muk)$.
\begin{lemma}
\label{lem:auxdif}
Let $\W$, $\W'$,$\U$, $\U'$ satisfy $\norm{\W} , \norm{\W'} \leq \lambda$, and $\norm{\U} , \norm{\U'} \leq \lambda$. Let $\htn(\thk,\muk)$ and $\htn(\thk', \muk')$ be the state vectors obtained at time $t$ by running the model in (\ref{model1})-(\ref{model2}) with the matrices $\W$, $\U$, and $\W'$, $\U'$ on common input sequence $\{\textbf{x}_1,\textbf{x}_2,\cdots,\textbf{x}_{t}\}$, respectively. If $\textbf{h}_0(\thk, \muk)=\textbf{h}_0(\thk', \muk')$, then
\begin{equation}
\label{eq:dif}
\norm{\htn(\thk, \muk) - \htn(\thk', \muk')} \! \leq \! \sum_{i=0}^{t-1} \lambda^i \big( \sqrt{\nh} \norm{\thk-\thk'} + \sqrt{\nx}\norm{\muk-\muk' } \big).
\end{equation}    
\end{lemma}

\begin{proof}
See the Appendix.
\end{proof}
\begin{remark}
\label{rem:11}
We note that, by (\ref{eq:dif}), to ensure $\htn(\thk,\muk)$  has a bounded gradient with respect to $\thk$ and $\muk$ in an infinite time horizon, $\lambda$ should be in $[0,1)$, i.e., $\lambda \in [0,1)$. In this case, the right hand side of (\ref{eq:dif}) becomes bounded, i.e.,
\begin{equation}
\norm{\htn(\thk, \muk) - \htn(\thk', \muk')} \! \leq \! \frac{\sqrt{\nh}}{1-\lambda} \norm{\thk-\thk'} + \frac{\sqrt{\nx}}{1-\lambda} \norm{\muk-\muk' }
\end{equation}    
for any $t \in [T]$. 

We recall that $L_{t_w}(\thk,\muk)$ is dependent on $\htn(\thk,\muk)$ due to (\ref{def:loss}) and (\ref{def:tsloss}). Hence, to ensure the derivatives of $L_{t_w}(\thk,\muk)$ stay bounded, we need to constrain our parameter space as $\mathcal{K}_\theta=\{ \vect(\W) : \norm{\W} \leq \lambda \}$ and $\mathcal{K}_\mu=\{ \vect(\U) : \norm{\U} \leq \lambda \}$ for some $\lambda \in [0,1)$. Note that since $\Kt$ and $\Km$ are convex sets for any $\lambda \in [0,1)$, our constraint does not violate the setting described in the previous subsection.
\end{remark}

Now that we have found $\lambda \in [0,1)$ is sufficient for $L_{t,w}(\thk,\muk)$ to have bounded derivatives in any $t \in [T]$, in the following theorem, we provide its  smoothness constants.

\begin{theorem}
\label{th:lips}
Let $\thk = \vect(\W)$ and $\muk = \vect(\U)$ , where $\W$ and $\U$ satisfy $\norm{\W} \leq \lambda$, and $\norm{\U} \leq \lambda$ for some $\lambda \in [0,1)$. Then, $L_{t,w}(\thk,\muk)$ has the following Lipschitz and smoothness properties:
\begin{align}
&1 ) \: \bnorm{ \frac{\del^2 \! L_{t,w}(\thk,\muk)}{\del \thk^2}} \leq \beta_{\theta}, \textrm{ where } \bTh = \frac{4 \nh \sqrt{\nh}}{(1-\lambda)^3}. \label{th1:3} \\
&2 ) \: \bnorm{ \frac{\del^2 \! L_{t,w}(\thk,\muk)}{\del \muk^2}} \leq \beta_{\mu}, \textrm{ where } \bMu = \frac{4 \nx \sqrt{\nh}}{(1-\lambda)^3}. \label{th1:4} \\
&3 ) \: \bnorm{ \frac{\del^2 \! L_{t,w}(\thk,\muk)}{\del \thk \del \muk} } \! \leq \! \bThMu, \textrm{ where } \bThMu \! = \! \frac{4 \nh  \sqrt{\nx}}{(1-\lambda)^3}. \label{th1:5}
\end{align}
\end{theorem}

\begin{proof}
See the Appendix.
\end{proof}

In the following section, we use these properties to develop an SRNN training algorithm with a convergence guarantee.

\subsection{Main Algorithm}
\begin{algorithm}[t!]
	\fontsize{9}{16}
	\caption{Windowed Online Gradient Descent Algorithm (WOGD)}\label{alg:alg1}
	\begin{algorithmic}[1]
	    \STATE \textbf{Parameters:} \begin{itemize}
	    								\item Learning rate $\eta \in [0,1)$
	    							    \item Window-size $w \in [T]$ 
	    							    \item $\lambda \in [0,1)$ 
	    							  \end{itemize} \label{alg1:param}
		\STATE Initialize $\thk_1$, $\muk_1$, $\C_1$ and $\textbf{h}_0$. \label{alg1:init}
		\STATE Let
			   \begin{itemize}
	    		      \item $\Kt=\{ \vect(\W) : \norm{\W} < \lambda \}$
	    			  \item $\Km=\{ \vect(\U) : \norm{\U} \leq \lambda \}$
	    		      \item $\Kcc= \{ \C : \norm{\C} \leq 1 \}$
	    		    \end{itemize} \label{alg1:def}					      
		\FOR{$t=1$ \TO $T$}
		\STATE Predict $\thk_t$, $\muk_t$ and $\C_t$. \label{alg1:pred}
		\STATE Receive $\xt$ and generate $\dht$. \label{alg1:forward}
		\STATE Observe $\dt$ and the cost function $\ell_t(\thk_t,\muk_t)$. \label{alg1:obs}
		\STATE Updates: \label{alg1:upd}
		\begin{align}
		&\C_{t+1} = \Pi_{\Kcc} \Big[ \C_t - \frac{1}{\sqrt{t}} \frac{\del \slossgent}{\del \Ct} \Big] \label{eq:updc}  \\
		&\thk_{t+1}= \thk_t - \eta \projLTht  \label{eq:updth}\\
		&\muk_{t+1}= \muk_t - \eta \projLMt. \label{eq:updmu}
		\end{align}		
		\ENDFOR
	\end{algorithmic}
\end{algorithm}

In this part, we present our algorithm, namely the Windowed Online Gradient Descent Algorithm (WOGD), shown in Algorithm \ref{alg:alg1}. 

In the algorithm, we take the learning rate $\eta \in [0,1)$,  window-size $w \in [T]$ and $\lambda \in [0,1)$ as the inputs. We, then, define the parameter spaces $\Kt$, $\Km$, and $\Kcc$ in line \ref{alg1:def}. Here, we define $\Kt$ and $\Km$ as given in Remark \ref{rem:11} to ensure that the derivatives of the loss functions are bounded. Furthermore, we define $\Kcc$ as $\Kcc= \{ \C : \norm{\C} \leq 1 \}$  to satisfy our assumption of $\norm{\C} \leq 1$. 

In the learning part, we first predict the hidden layer weight matrices, i.e., $\thk_t$ and $\muk_t$, and the output layer weights, i.e., $\C_t$ (see line \ref{alg1:pred}). Then, we receive the input vector $\xt$ and generate our estimate $\dht$ by running the model in (\ref{model1})-(\ref{model2}). We next observe ground truth value $\dt$ and the loss function $\ell_t(\thk_t,\muk_t)$ in line \ref{alg1:obs}. Having observed the label, we update the weight matrices in line \ref{alg1:upd} (or in (\ref{eq:updc})-(\ref{eq:updmu})). Here, we update the output layer weights $\Ct$ with the projected online gradient descent algorithm \cite{ZinkOGD}. We update the hidden weights in (\ref{eq:updth})-(\ref{eq:updmu}) by using the projected partial derivatives of  the time-smoothed loss function $\slossgent$ defined with $(\Kt,\eta)$ and $(\Km,\eta)$. 

Note that since WOGD optimize the network weights with the time-smoothed loss function (rather than instantaneous loss), it can be intuitively interpreted as SGD with a windowed momentum, as the gradients used to perform updates at iteration $t$ depends on the loss suffered in the past past time steps. However, as opposed to the standard momentum-based approaches, such as Adam and RMSprop, which accumulate the gradient information of the previous time steps with exponential weighting, WOGD repeatedly calculates the gradient of the previous time steps in its window with respect to the network weights at time $t$ at each iteration.

As a further note, since we constructed our setting for adversarial $\Ct$ selections, the update rule for the output layer in (\ref{eq:updc}) does not contradict with our analysis. Moreover, since $\slossgent$ is the average of last $w$ losses, which are all convex with respect to  $\Ct$,  $\slossgent$ is also convex with respect to $\Ct$. Then, by using \cite[Theorem 1]{ZinkOGD}, we can prove the update rule in (\ref{eq:updc}) converges to the best possible output layer weights satisfying $\norm{\C} \leq 1$. Therefore, in the following theorem, we provide the convergence guarantee of WOGD specifically for the hidden layer weights. 
\begin{theorem}
\label{th:conv}
Let $\loss$ and  $\slossgen$ be the loss and time-smoothed loss functions defined in (\ref{def:loss}) and (\ref{def:tsloss}), respectively. Moreover, let  $\beta$ be the maximum possible smoothness constant, i.e.,
\begin{equation}
\label{eq:beta}
\beta = \max\{\beta_\theta, \beta_\mu, \bThMu\},
\end{equation}
where $\beta_\theta$, $\beta_\mu$,  and $\bThMu$ are defined in (\ref{th1:3}), (\ref{th1:4}) and (\ref{th1:5}). Then, if WOGD is run with the parameters
\begin{equation}
\label{eq:params1}
0 < \eta \leq \frac{1}{\beta}
\end{equation}
it ensures that
\begin{equation}
\label{eq:regbound}
R_w(T) \leq \frac{16 \sqrt{\nh}}{\eta} \frac{T}{w} + \frac{16 \sqrt{\nh}}{\eta},
\end{equation}
where $R_w(T)$ is the local regret defined in (\ref{def:locreg}). By selecting a window-size $w$ such that $\frac{T}{w}= o(T)$, one can bound $R_w(T)$ with a sublinear bound, hence, guarantee convergence of the hidden layer weights to the locally optimum parameters.\footnote{We use  little-o notation, i.e., $g(x)= o(f(x))$, to describe an upper-bound that cannot be tight, i.e., $\lim_{ x \to \infty} g(x)/f(x) =0$.}
\end{theorem}
\begin{proof}
See the Appendix.
\end{proof}

Theorem \ref{th:conv} shows that with appropriate parameter selections, WOGD guarantees to learn the locally optimum SRNN parameters for any bounded input/output sequences. Moreover, here, we observe that the window-size of the algorithm directly controls its convergence rate (see (\ref{eq:regbound})). That will the key property of our algorithm to obtain the regression performance of LSTMs with SRNNs.

Now that we have proved the convergence guarantee of WOGD, in the following remark, we investigate the computational requirement of WOGD.

\begin{remark}
\label{rem:compreq}
The most expensive operation of WOGD is the update rule of the hidden layer weights, i.e., (\ref{eq:updth})-(\ref{eq:updmu}), which can also be written in the form as (see (\ref{def:projTh})-(\ref{def:projM}))
\begin{alignat*}{2}
\bm{\hat{\theta}}_{t+1} &= \thk_t - \eta \partLTht  &&\textrm{ - Update} \\
\thk_{t+1} &= \Pi_{\Kt} \Big[ \bm{\hat{\theta}}_{t+1}  \Big] &&\textrm{ - Projection}.
\end{alignat*}
Note that the update part of requires the computation of the partial derivatives of $\slossgen$ with respect to $\thk$ and $\muk$. Additionally, the projection part projects $\bm{\hat{\theta}}_{t+1}$ onto the spectral unit ball $\Kt$, which requires a matrix projection operation. 

To compute the partial derivatives, we use the Truncated Backpropagation Through Time algorithm \cite{Will95}, which has $\bOh{h \nh ( \nh + \nx)}$ computational complexity with a truncation length $h$. Since WOGD uses the partial derivatives of the last $w$ losses, we can approximate to these partial derivatives with a single back-propagation by using a truncation length $w$, which results in $\bOh{w \nh ( \nh + \nx)}$ computational requirement for computing the partial derivatives. 

Furthermore, the projection step can be written as
\begin{equation}
\thk_{t+1} = \argmin_{ \thk \in \Kt} \norm{\thk - \bm{\hat{\theta}}_{t+1} }. \label{eq:proj1}
\end{equation}
By \cite[Proposition 9]{SedghiGL19}, (\ref{eq:proj1}) can be performed by computing the singular value decomposition (SVD) of the matrix form of  $\bm{\hat{\theta}}_{t+1}$ and clipping its singular values with $\lambda$. 

Note that the projection step can be computationally demanding due to computational requirement of SVD. However, it is commonly observed in the literature that when the weights are initialized close to the origin, they generally stay close to the origin during the training~\cite{Martin18implicit,Arora19}. By using this observation,  we can reduce the computational requirement of the projection part by performing projection only when the $\ell_2$ norm of the weights exceed some predetermined threshold $\alpha$, i.e.,
\begin{equation}
\label{eq:ifstat}
\thk_{t+1} = \begin{cases} 
      \argmin_{ \thk \in \Kt} \norm{\thk - \bm{\hat{\theta}}_{t+1} } & \norm{\bm{\hat{\theta}}_{t+1}} > \alpha \\
      \bm{\hat{\theta}}_{t+1} &   \norm{\bm{\hat{\theta}}_{t+1}} \leq \alpha.
   \end{cases}
\end{equation}
Since the weights are expected to stay small during the training, (\ref{eq:ifstat}) is expected to ensure the stability of the weight matrices with a reasonable $\alpha$ selection, such as $\alpha \in [5,10]$, efficiently. Indeed, during the experiments, we observe that when the hidden size is chosen in a reasonable range, i.e., $5 \leq n_h \leq 20$, WOGD performs the projection step very rarely -at most $3$ times in a single simulation. Therefore, the main computational bottleneck of our algorithm is to compute the partial derivatives of $\slossgen$ with respect to $\thk$ and $\muk$, which requires  $\bOh{w \nh ( \nh + \nx)}$, i.e., a linear time complexity in the number of weights. With this complexity, our algorithm has the same computational requirement as SGD\cite{Will95}.
\end{remark}

In the next remark, we discuss the effect of choosing higher learning rate than the theoretically guaranteed one in Theorem 2.

\begin{remark}
\label{rem:lrs}
We note that WOGD is constructed by assuming the worst-case Lipschitz constants derived in Theorem \ref{th:lips}. On the other hand, our experiments suggest that in practice, the landscape of the objective function is generally nicer than what is predicted by our theoretical development. For example, in the simulations, we observe that the smoothness of the error surface is usually $10^4$ to $10^5$ times smaller than their theoretical upper-bounds given in (\ref{th1:3})-(\ref{th1:4}). Therefore, it is practically possible to obtain vanishing regret with WOGD by using much higher learning rate than the theoretically guaranteed one.  Since the regret bound of WOGD is inversely  proportional with the learning rate (see (\ref{eq:regbound})), in the following, we use WOGD with the higher learning rates than suggested in Theorem \ref{th:conv} to obtain faster convergence.
\end{remark}

\section{Experiments}\label{sec:sim}

In this section, we verify our theoretical results and demonstrate the performance improvements of our algorithm. To this end, we use four real-life and two synthetic datasets. We compare our algorithm with three widely used neural-network training algorithms: Adam, RMSprop, and SGD. To illustrate performance differences between the models, we apply Adam, RMSprop, and SGD to SRNNs, LSTMs, and Clockwork RNNs (CWRNNs). Since the main focus of our paper is to obtain the regression performance of the state-of-the-art models with SRNNs, we apply WOGD only to SRNNs. In the following, we use  the prefixes ``SRNN-", ``LSTM-", and ``CWRNN-" to denote to which model the optimization algorithm is applied, such as CWRNN-Adam, LSTM-RMSprop, or SRNN-WOGD. 

In the experiments, we use the most widely used LSTM model, where the activation functions are set to the hyperbolic tangent function, and the peep-hole connections are eliminated\cite{Tolga2018}. We implement CWRNN as instructed in its original publication, and use the exponential clock-timings $\{1,2,4,8,16\}$ in the Real-Life Datasets part, and $\{1,2,4,8\}$ in the Binary Addition part. As the SRNN model, we use the standard Elman model given in (\ref{model1})-(\ref{model2}). Moreover, we implement our algorithm slightly differently than Algorithm \ref{alg:alg1} to obtain its maximum performance. Differing from Algorithm \ref{alg:alg1}, we use $8/\sqrt{t}$ as the output layer learning rate and choose the maximum $\ell_2$ norm of the output layer weights as $2.5$, i.e., $\norm{\C_t} \leq 2.5$ for $t \in [T]$. We emphasize that these changes do not hurt our convergence guarantee as they alter the regret bound in Theorem \ref{th:conv} only by a constant factor.

In all simulations, we randomly draw the initial weights from a Gaussian distribution with zero mean and standard deviation of $0.1$. In all SRNN-WOGD runs, we use $\alpha = 7.5$ and $\lambda = 0.95$. We search the hyperparameters of the learning algorithms over a dense grid.  For each hyperparameter in each optimization method, we repeat the learning procedure ten times with randomly chosen initial parameters (generated with fixed and different seeds), and report the results using the hyperparameters that minimize the mean of the mean squared errors obtained in the ten runs. We run each experiment thirty times with randomly chosen initial parameters (generated with fixed and different seeds), and provide the mean performance of the algorithms.

\subsection{Real-Life Datasets}
\begin{figure*}[t]
\begin{subfigure}[t]{0.5\textwidth}
   \centering
\includegraphics[width=0.9\textwidth]{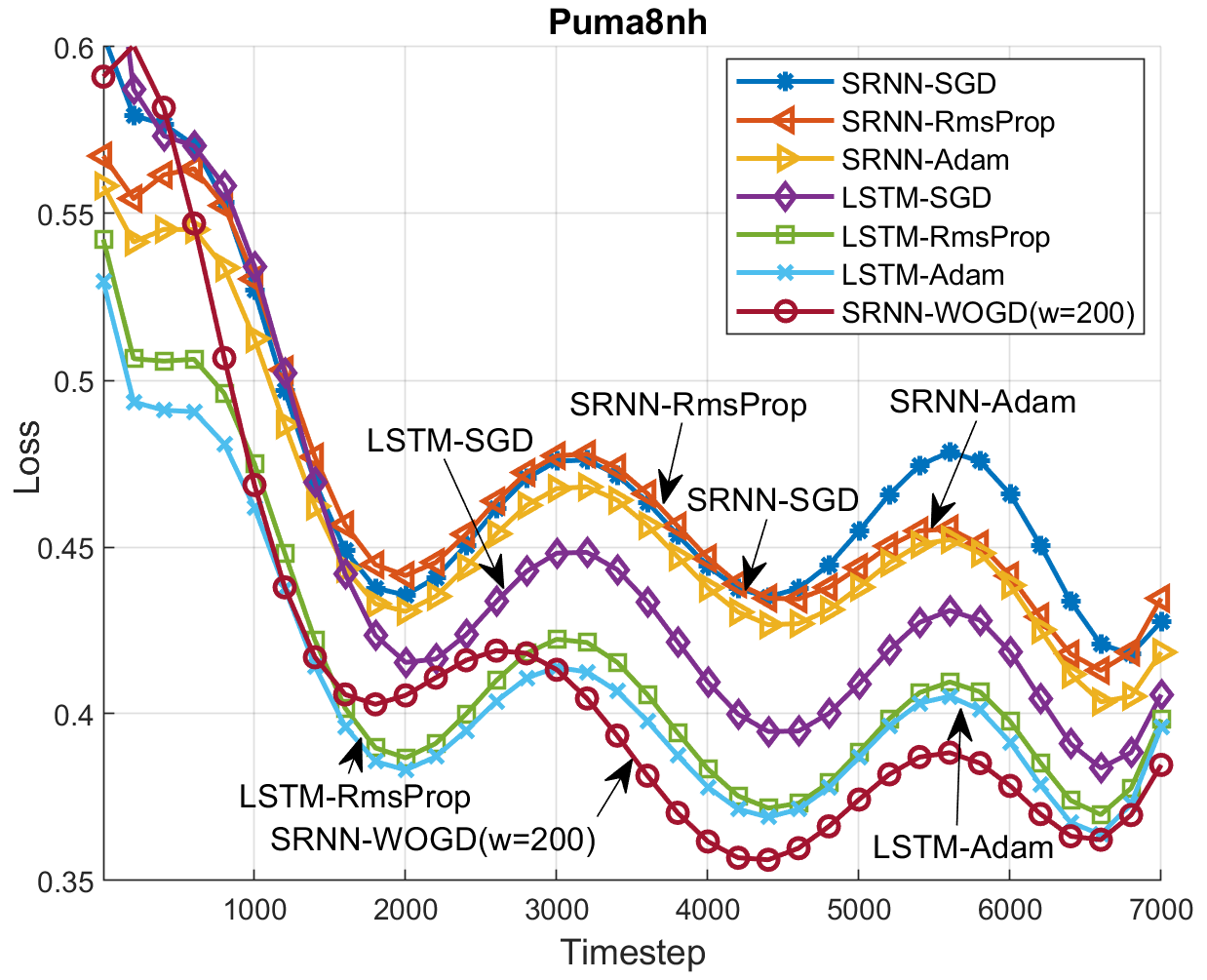}
\caption{} \label{fig1:r1}
\end{subfigure}
 \begin{subfigure}[t]{0.5\textwidth}
   \centering
\includegraphics[width=0.9\textwidth]{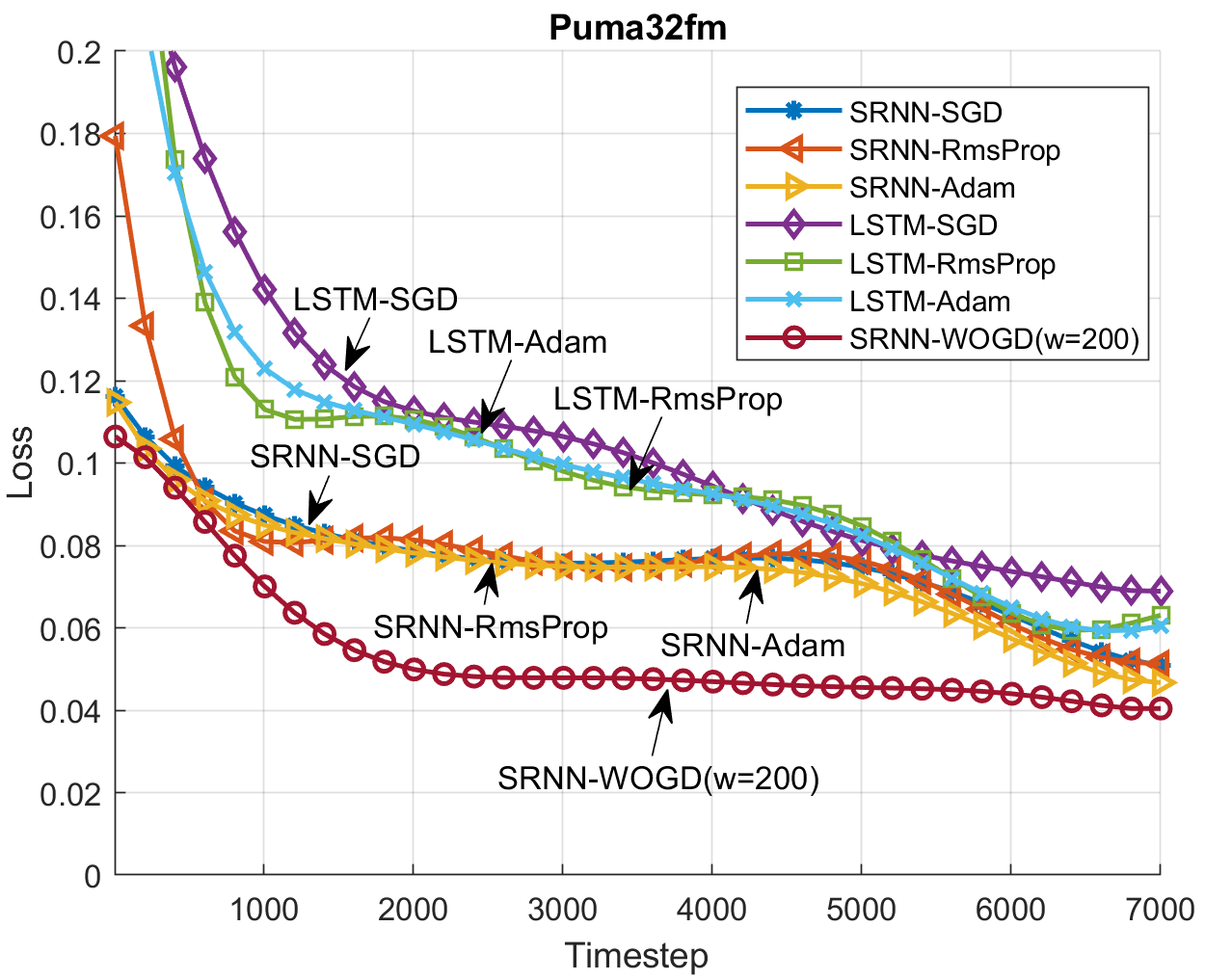}
\caption{} \label{fig1:r2}
\end{subfigure}  \\
\begin{subfigure}[t]{0.5\textwidth}
   \centering
\includegraphics[width=0.9\textwidth]{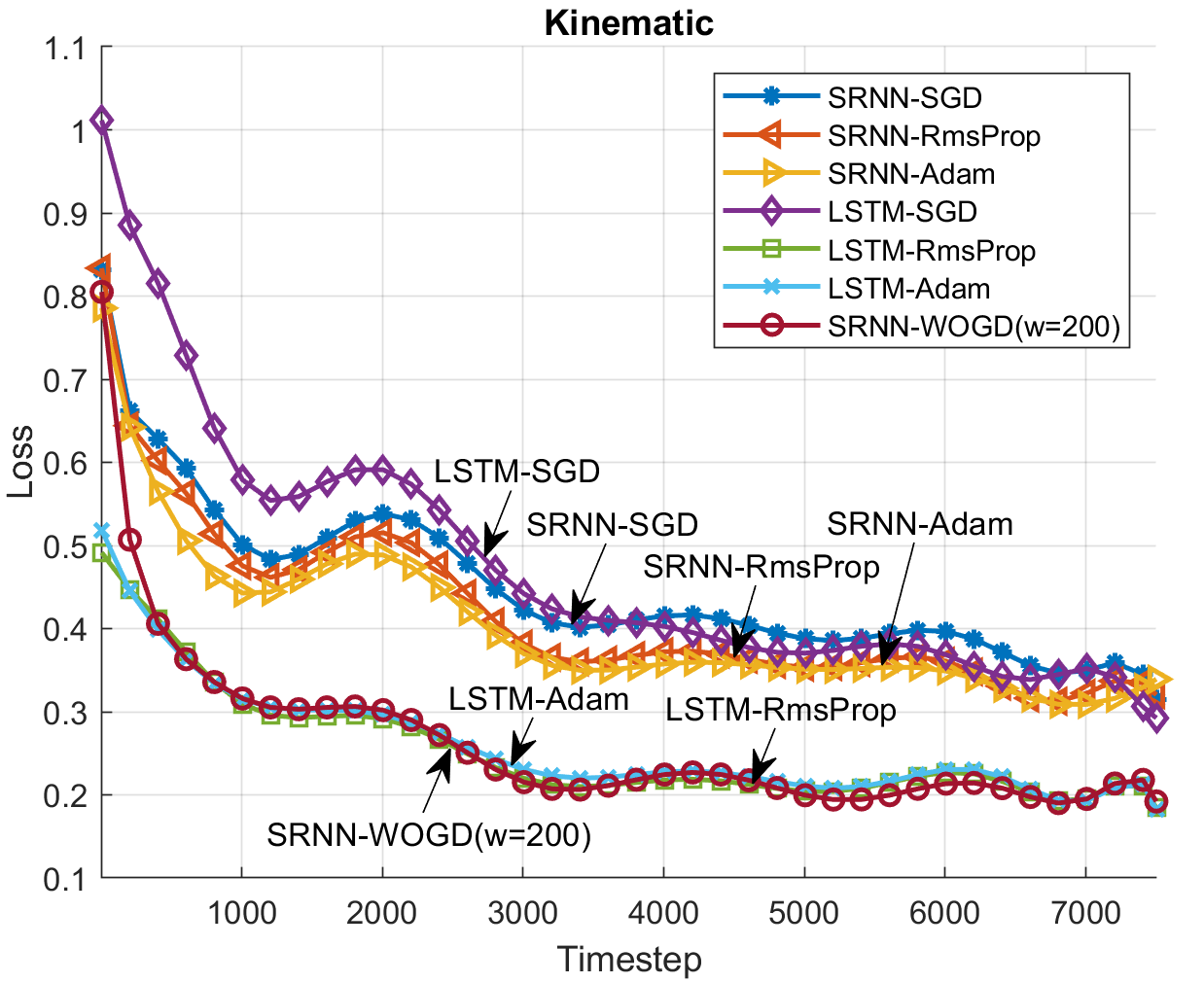}
\caption{} \label{fig1:r3}
\end{subfigure}
\begin{subfigure}[t]{0.5\textwidth}
   \centering
\includegraphics[width=0.9\textwidth]{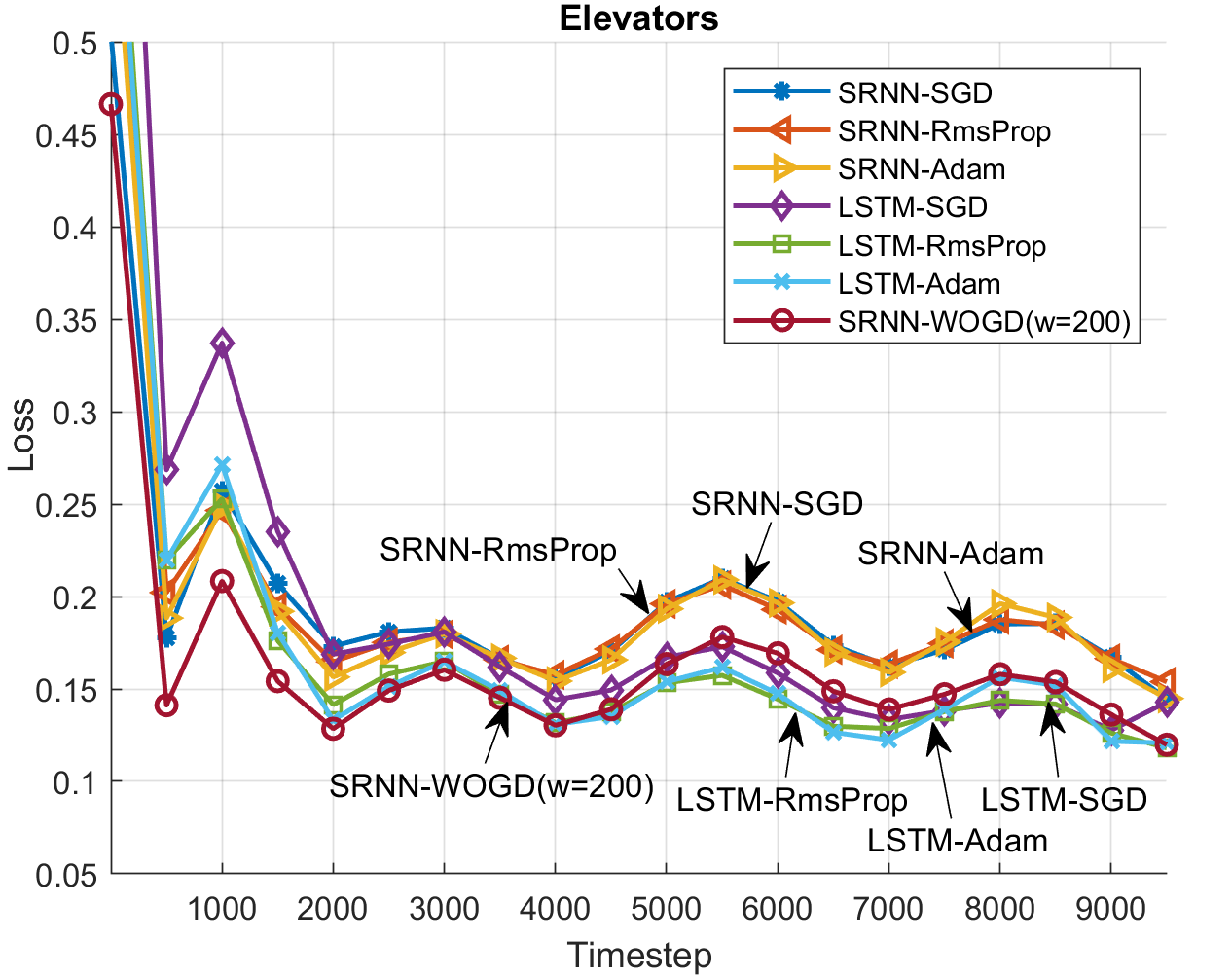}
\caption{} \label{fig1:r4}
\end{subfigure}
\caption{Sequential prediction performances of the algorithms on the (a) puma8nh, (b) puma32fm, (c) kinematic, and (d) elevators datasets. For each model, the learning curve of the training algorithm with the minimum mean squared error is plotted.}\label{fig:1}
\end{figure*}

\begin{table*}[thbp]
  \centering
    \begin{tabular}{|l|c|c|c|c|c|c|c|c|}
    \toprule
    \multicolumn{1}{|p{9.84em}|}{Datasets} & \multicolumn{2}{c|}{Puma8nh} & \multicolumn{2}{c|}{Puma32fm} & \multicolumn{2}{c|}{Kinematics} & \multicolumn{2}{c|}{Elevators} \\
    \midrule
    \multicolumn{1}{|c|}{Algorithms} & MSE   & Run-time(s) & MSE   & Run-time(s) & MSE   & Run-time(s) & MSE   & Run-time(s) \\
    \midrule[1.25pt]
    SRNN-WOGD(w = 50) & 0.456 & 1.32  & 0.068 & 1.40  & 0.346 & 1.31  & 0.183 & 1.72 \\
    \midrule
    SRNN-WOGD(w = 100) & 0.428 & 2.14  & 0.060 & 2.29  & 0.293 & 2.06  & 0.168 & 2.80 \\
    \midrule
    SRNN-WOGD(w = 200) & \textbf{0.408} & \textbf{4.16} & \textbf{0.053} & \textbf{4.66} & \textbf{0.263} & \textbf{4.08} & \textbf{0.158} & \textbf{5.42} \\
    \midrule[1.25pt]
    SRNN-SGD & 0.471 & 0.87  & 0.075 & 0.95  & 0.448 & 0.83  & 0.186 & 1.09 \\
    \midrule
    SRNN-RMSprop & 0.467 & 0.93  & 0.076 & 0.98  & 0.419 & 0.88  & 0.188 & 1.15 \\
    \midrule
    SRNN-Adam & 0.451 & 0.98  & 0.072 & 1.05 & 0.407 & 0.94  & 0.186 & 1.28 \\
    \midrule[1.25pt]
    CWRNN-SGD & 0.469 & 5.62  & 0.065 & 5.64  & 0.426 & 5.79  & 0.185 & 7.18 \\
    \midrule
    CWRNN-RMSprop & 0.453 & 5.48  & 0.066 & 5.62  & 0.386 & 5.73  & 0.184 & 7.21 \\
    \midrule
    CWRNN-Adam & 0.445 & 5.60  & \textbf{0.064} & \textbf{5.80} & 0.383 & 5.82  & 0.181 & 7.54 \\
    \midrule[1.25pt]
    LSTM-SGD & 0.453 & 7.32  & 0.111 & 7.77  & 0.316 & 7.33  & 0.189 & 9.65 \\
    \midrule
    LSTM-RMSprop & 0.412 & 7.80  & 0.101 & 7.85  & \textbf{0.261} & \textbf{7.55} & 0.162 & 10.13 \\
    \midrule
    LSTM-Adam & \textbf{0.408} & \textbf{7.60} & 0.100 & 7.97  & 0.265 & 7.91  & \textbf{0.160} & \textbf{10.10} \\
    \bottomrule
    \end{tabular}%
    \caption{Mean squared errors (MSE) and the corresponding run-times (in seconds) of the compared algorithms. The two algorithms with the lowest two mean errors are emphasized. The simulations are performed on a computer with i7-7500U processor, 2.7-GHz CPU, and 8-GB RAM.}
  \label{tab:perf1}%
\end{table*}%

\begin{figure*}[t]
\begin{subfigure}[t]{0.5\textwidth}
   \centering
\includegraphics[width=0.95\textwidth]{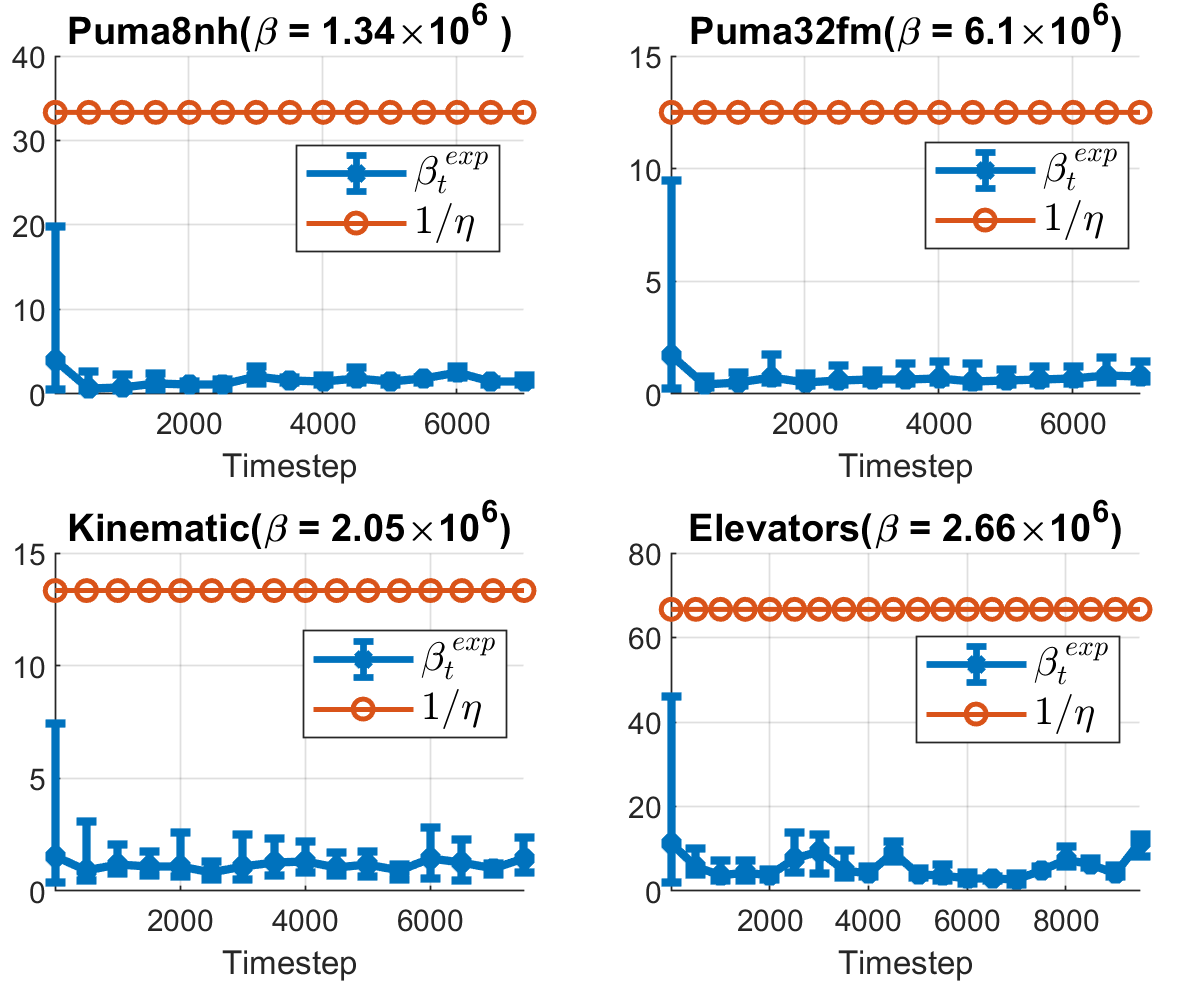}
\caption{} \label{fig2:r1}
\end{subfigure}
 \begin{subfigure}[t]{0.5\textwidth}
   \centering
\includegraphics[width=0.95\textwidth]{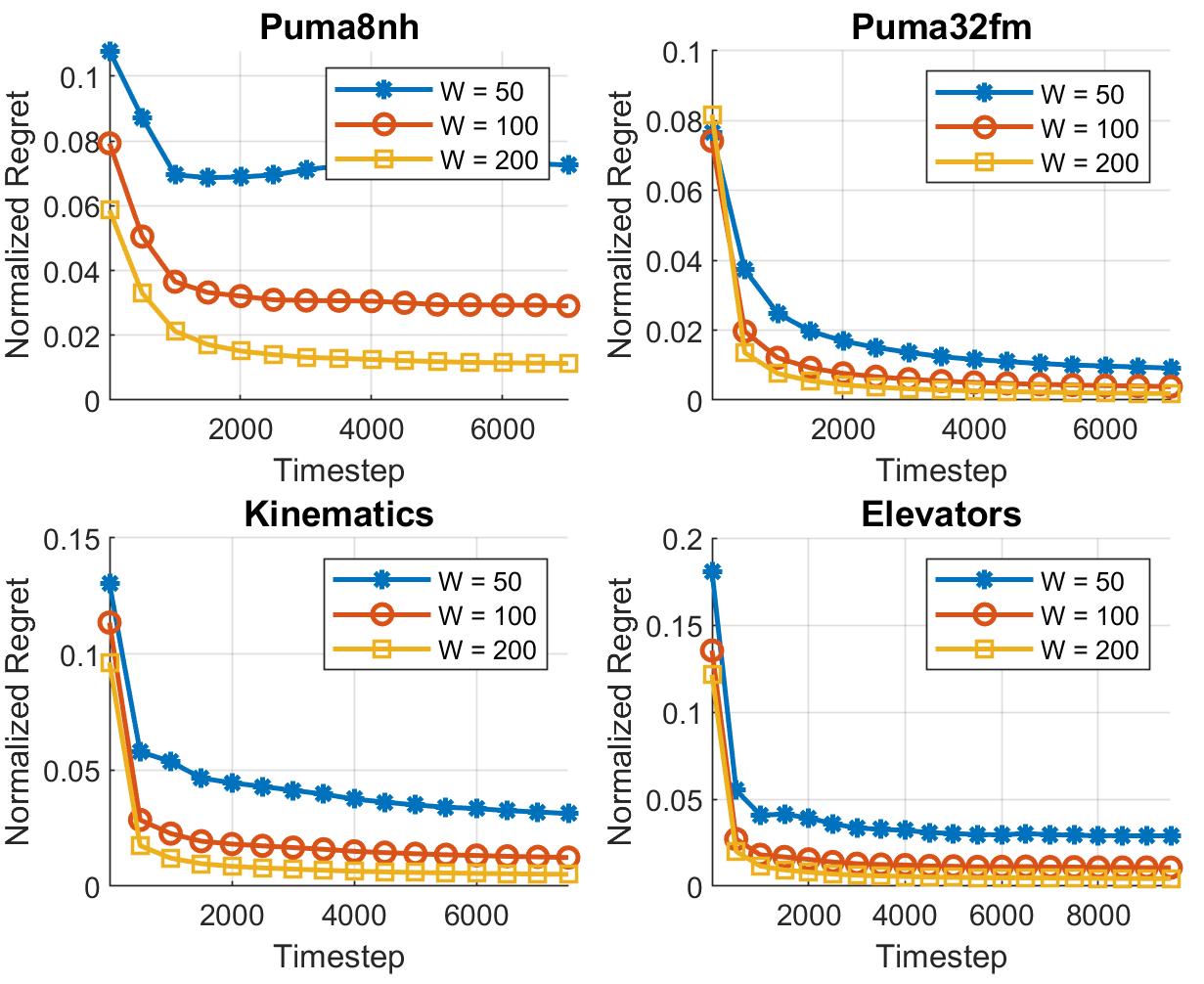}
\caption{} \label{fig2:r2}
\end{subfigure} 
\caption{(a) Comparison between smothnesses the error surfaces and their theoretical upper-bounds formulated in (\ref{eq:beta})-(\ref{eq:params1}). (b) The normalized regret bounds of SRNN-WOGD, i.e., $R_w(t)/t$ for $t \in [T]$, with varying window-sizes.}\label{fig:2}
\end{figure*}

\begin{table*}[t]

\begin{subtable}{\linewidth}
\centering
\begin{tabular}{@{}|c|c|c|c|c|c|c|c|c|@{}}
\toprule
Algorithms & \multicolumn{2}{c|}{RNN-WOGD(w=200)}   & \multicolumn{2}{c|}{LSTM-RMSprop}      & \multicolumn{2}{c|}{CWRNN-RMSprop}   & \multicolumn{2}{c|}{RNN-RMSprop}       \\ \midrule
Net        & Timestep & Run-Time(s)   & Timestep & Run-Time(s) & Timestep & Run-Time(s) & Timestep & Run-Time(s)   \\ \midrule
1          & \textbf{1911}          & 0.73          & 2787                     & 2.23        & 9830                   & 3.58        & 5325                   & \textbf{0.51} \\ \midrule
2          & \textbf{2050}          & \textbf{0.85} & 2828                     & 2.24        & 7189                   & 3.88        & 9535                   & 0.93          \\ \midrule
3          & \textbf{1454}          & \textbf{0.56} & 2588                     & 1.99        & 17298                  & 4.85        & 14331                  & 1.39          \\ \midrule
4          & \textbf{1934}          & \textbf{0.79} & 3799                     & 2.87        & 9336                   & 4.56        & 9079                   & 0.87          \\ \midrule
5          & \textbf{1891}          & \textbf{0.74} & 2607                     & 2.03        & 9104                   & 3.12        & 10125                  & 0.96          \\ \bottomrule
\end{tabular}
\caption{Adding two binary sequences}
\label{tab2:r1}
\end{subtable}
\vspace{1mm}

\begin{subtable}{\linewidth}
\centering
\begin{tabular}{@{}|c|c|c|c|c|c|c|c|c|@{}}
\toprule
Algorithms & \multicolumn{2}{c|}{RNN-WOGD(w=200)}    & \multicolumn{2}{c|}{LSTM-RMSprop}         & \multicolumn{2}{c|}{CWRNN-RMSprop}   & \multicolumn{2}{c|}{RNN-RMSprop}       \\ \midrule
Net        & Timestep & Run-Time(s)    & Timestep & Run-Time(s)    & Timestep & Run-Time(s) & Timestep & Run-Time(s) \\ \midrule
1          & \textbf{19398}         & 9.53           & 22553                    & 16.97          & \multicolumn{2}{c|}{Failed}          & \multicolumn{2}{c|}{Failed}            \\ \midrule
2          & \textbf{18891}         & \textbf{9.03}  & 30062                    & 23.48          & \multicolumn{2}{c|}{Failed}          & \multicolumn{2}{c|}{Failed}            \\ \midrule
3          & 27173                  & 11.99          & \textbf{15171}           & \textbf{11.42} & \multicolumn{2}{c|}{Failed}          & \multicolumn{2}{c|}{Failed}            \\ \midrule
4          & \textbf{21499}         & \textbf{10.46} & 42526                    & 32.38          & \multicolumn{2}{c|}{Failed}          & \multicolumn{2}{c|}{Failed}            \\ \midrule
5          & \textbf{23718}         & \textbf{11.64} & 25623                    & 19.57          & \multicolumn{2}{c|}{Failed}          & \multicolumn{2}{c|}{Failed}            \\ \bottomrule
\end{tabular}
\caption{Adding three binary sequences}
\label{tab2:r2}
\end{subtable}

\caption{Timesteps and the run-times (in seconds) required to achieve $1000$ subsequent error-free predictions, i.e., sustainable prediction. The experiments are repeated with five different input streams and the results are presented in order. The bold font shows the best result (in terms of both timestep and run-time) for each input stream. The simulations are performed on a computer with i7-7500U processor, 2.7-GHz CPU, and 8-GB RAM.} 
\end{table*}

\subsubsection{Pumadyn Datasets}
In the first part,  we consider the pumaydn dataset\cite{Puma}, which includes $7000$ input/output pairs obtained from the simulation of Unimation Puma $560$ robotic arm, i.e., $T=7000$. Here, we aim to estimate the angular acceleration of the arm by using the angular position and angular velocity of the links.  To compare the algorithms under various scenarios, we use two variants of the pumaydn dataset with different difficulties, i.e., \textit{puma8nh}, which is generated with nonlinear dynamics and high noise, and \textit{puma32fm}, which is generated with fairly linear dynamics and moderate noise. 

For the puma8nh dataset,  we use $8$-dimensional input vectors of the dataset with an additional bias dimension, i.e., $\nx=9$, and $10$-dimensional state vectors, i.e., $\nh=10$. In SRNN-Adam, SRNN-RMSprop, and SRNN-SGD, we use the learning rates of $0.005$, $0.007$, and $0.03$, respectively. In CWRNN-Adam, CWRNN-RMSprop, and CWRNN-SGD, we use the learning rates of $0.007$, $0.009$, and $0.03$. For LSTM-Adam, LSTM-RMSprop, and LSTM-SGD, we choose the learning rates as  $0.01$, $0.01$, and $0.007$. In SRNN-WOGD, we use $\eta = 0.03$. To test the effect of the window-size on performance, we run SRNN-WOGD with three different window-sizes, i.e., $w \in \{50,100,200\}$.

We plot the learning curves of the puma8nh experiment in Fig. \ref{fig1:r1}. Here, we see that the LSTM-based methods and SRNN-WOGD(w = 200) outperform the SRNN and CWRNN based state-of-the-art methods while providing very similar performances. We present the mean errors and the run-times of this part in the first column of Table \ref{tab:perf1}. In the table, we observe that as consistent with our theoretical results, the performance of SRNN-WOGD improves as its window-size, i.e., $w$, gets larger. Moreover, we see that LSTM-RMSprop and SRNN-WOGD(w = 200) are on a par error-wise, while SRNN-WOGD(w = 200) achieves the equivalent performance in almost two times shorter run-time. 

For the puma32fm dataset,  we use $32$-dimensional input vectors of the dataset with an additional bias dimension, i.e., $ \nx = 33$, and $10$-dimensional state vectors, i.e., $ \nh=10$. In SRNN-Adam, SRNN-RMSprop, and SRNN-SGD, we choose the learning rates as $0.001$, $0.001$, and $0.01$. In CWRNN-Adam, CWRNN-RMSprop, and CWRNN-SGD, we use the learning rates of $0.002$, $0.002$, and $0.03$. For LSTM-Adam, LSTM-RMSprop, and LSTM-SGD, we use the learning rates of  $0.002$, $0.002$, and $0.065$. In SRNN-WOGD, we use $\eta = 0.08$ and $w \in \{50,100,200\}$.

We plot the learning curves of the puma32fm experiment in Fig. \ref{fig1:r2}. In the figure, we see that SRNN-WOGD(w = 200) converges to small loss values much faster than the other algorithms.  Additionally, we observe that the LSTM-based methods converge to the same loss values as the SRNN and CWRNN based models, yet more slowly due to their higher number of parameters.  We present the mean errors and the run-times of this part in the second column of Table \ref{tab:perf1}.  Here, we see that as in the previous experiment, the error of SRNN-WOGD reduces as its window size increases. Moreover, we observe that SRNN-WOGD(w = 200) provides a considerably smaller mean error compared to the other models due to its relatively fast convergence.

\subsubsection{Kinematic and Elevators Datasets}
In the second part,  we consider the kinematic and elevators datasets, which include $7500$ and $9500$ input/output pairs, respectively\cite{kin,elev}. The kinematic dataset is obtained from a simulation of an eight-link all-revolute robotic arm, and the aim is to predict the distance of the effector from a target. The elevators dataset is obtained from a procedure related to controlling an F16 aircraft, and the aim is to predict the variable that expresses the actions of the aircraft.

For the kinematic dataset,  we use $8$-dimensional input vectors of the dataset with an additional bias dimension, i.e., $ \nx =9$, and $15$-dimensional state vectors, i.e., $ \nh=15$. In SRNN-Adam, SRNN-RMSprop, and SRNN-SGD, we choose the learning rates as $0.007$, $0.007$, and $0.035$. In CWRNN-Adam, CWRNN-RMSprop, and CWRNN-SGD, we use the learning rates of $0.009$, $0.01$, and $0.035$. For LSTM-Adam, LSTM-RMSprop, and LSTM-SGD, we use the learning rates of  $0.009$, $0.01$, and $0.15$. In SRNN-WOGD, we use $\eta = 0.075$ and $w \in \{50,100,200\}$.

For the elevators dataset,  we use $18$-dimensional input vectors of the dataset with an additional bias dimension, i.e., $\nx = 19$, and $15$-dimensional state vectors, i.e., $\nh = 15$. In both SRNN-Adam, SRNN-RMSprop, SRNN-SGD, and CWRNN-Adam, CWRNN-RMSprop, CWRNN-SGD we choose the learning rates as $0.002$, $0.002$, and $0.02$ respectively. For LSTM-Adam, LSTM-RMSprop, and LSTM-SGD, we use the learning rates of  $0.004$, $0.004$, and $0.05$. In SRNN-WOGD, we use $\eta = 0.04$ and $w \in \{50,100,200\}$.

We plot the learning curves of the kinematic and elevators experiments in Fig. \ref{fig1:r3} and Fig. \ref{fig1:r4}. Here again, we observe that the SRNNs trained with the state-of-the-art algorithms perform the worst, followed by the CWRNN-based models, LSTM-based models and SRNN-WOGD(w = 200). Moreover, we see that LSTM-based models and SRNN-WOGD(w=200) give substantially better results compared to other two models. We present the mean errors and the run-times of this part in the last two columns of  Table \ref{tab:perf1}. Here, we see that as in the previous experiments, the larger window-size SRNN-WOGD has, the lower mean error it attains. Moreover, we observe that in both experiments, SRNN-WOGD(w = 200) provides very similar mean error with the best LSTM models, yet within two times shorter training time due to efficiency improvements of our algorithm achieved by utilizing SRNNs.


\subsubsection{Experimental Verification of Theoretical Results}
In this section, we further verify our theoretical results by analyzing the smoothness of the error surface and behavior of the normalized regret in the simulations. To observe the smoothness parameters efficiently without calculating the Hessian matrix, we use the finite differences formulated as
\begin{equation*}
\beta^{exp}_t = \max\{\beta^{exp}_{\theta,t}, \beta^{exp}_{\mu,t}\},
\end{equation*} 
where
\begin{spreadlines}{0.8em}
\begin{align}
\beta^{exp}_{\theta,t} &= \frac{\bnorm{\partLThtp- \partLTht}}{\norm{\thk_{t+1} - \thk_t}}  \nonumber \\ 
\beta^{exp}_{\mu,t} &=  \frac{\bnorm{\partLMtp- \partLMt}}{\norm{\muk_{t+1} - \muk_t}}.\nonumber
\end{align}  
\end{spreadlines}
We note that since we use very small learning rates in SRNN-WOGD, the given finite differences closely approximate the smoothness of the error surface in the direction of the gradient update.

In Fig. \ref{fig2:r1}, we plot the smoothness parameters of the error surface that is obtained from the simulations of SRNN-WOGD(w = 200). In the plots, error bars extend between the minimum and maximum smoothness values observed in 30 simulations. Moreover, we provide the theoretical upper bounds (formulated in Theorem \ref{th:lips}) in the title of the plots. In the figure, we observe that as indicated in Remark \ref{rem:lrs}, the error surface is much smoother than what is predicted by its theoretical upper-bounds. We note that this gap is expected, since we derived the upper-bounds by considering the worst-case scenario, i.e., saturation region of SRNNs, which is rarely encountered in practice due to variations in real-world data. 

Additionally, we see in the plots that the selected learning rates satisfy the condition of Theorem \ref{th:conv} given in (\ref{eq:params1}). To verify the regret bound, we plot the normalized regret, i.e., $R_w(t)/t$ for $t \in [T]$, of all SRNN-WOGD runs  in Fig. \ref{fig2:r2}. Here, we see that as consistent with our theoretical derivation, the normalized regret vanishes and its convergence rate increases as the window-size gets larger, which agrees with the results of the previous four experiments.

\subsection{Binary Addition}
In this part, we compare the performance of the algorithms on a synthesized dataset that requires learning long-term dependencies. We show that SRNN-WOGD learns the long-term dependencies comparably well with LSTMs, which explains its success in the previous experiments.

To compare the algorithms, we construct a synthesized experiment in which we can control the length of temporal dependence. To this end, we train the network to learn the summation of $n$ number of binary sequences, where the carry bit is the temporal dependency that the models need to learn. We note that the number of added sequences, i.e., $n$, controls how long the carry bit is propagated on average, hence, the average length of the temporal dependence.  We learn binary addition with a purely online approach, i.e., there is only one single input stream, and learning continues even when the network makes a mistake.

In the experiments, we consider $n \in \{2, 3\}$. For each $n$, we repeat the experiments with $5$ different input streams. We generate the input sequences randomly, where $0$ and $1$'s are drawn with equal probabilities. In the models, we use the sigmoid function as the output layer activation function and the cross-entropy loss as the loss function. We assume that the network decides $1$ when its output is bigger than $0.5$, and $0$ in vice versa. For the performance comparison, we count the number of symbols needed to attain error-free predictions for $1000$ subsequent symbols, i.e., sustainable prediction. We note that as the models are more capable of learning long-term  dependencies, the number of steps required to obtain sustainable prediction is expected to be smaller.

In the first experiment, we consider adding $2$ sequences. The results are presented in Table \ref{tab2:r1}. Due to space constraints, we compare our algorithm only with RMSprop since we observe that RMSprop generally provides the fastest convergence in this task. In the table, we see that SRNN-WOGD(w = 200) consistently achieves the sustainable prediction  with considerably less number of time steps compared to SRNNs, CWRNNs and LSTMs. Moreover, due to the efficiency of our algorithm, SRNN-WOGD(w = 200) requires almost three times smaller run-time compared to LSTMs.

In the second experiment, we consider adding $3$ sequences. The results are presented in Table \ref{tab2:r2}.  Here, we see that CWRNN-RMSprop and SRNN-RMSprop fail in all simulations, i.e., they could not achieve the sustainable prediction within $5\times10^4$ timesteps. Moreover, as in the previous experiment, SRNN-WOGD(w = 200) obtains the sustainable prediction within a relatively small number of time steps and considerably shorter run-time compared to LSTMs in almost all experiments.

\section{Conclusion}\label{sec:concl}
We study online nonlinear regression with continually running SRNNs. For this problem, we introduce a first-order gradient-based optimization algorithm with a linear time complexity in the number of parameters, i.e., the same time complexity as the SGD algorithm. 

We construct our algorithm on a theoretical basis.  We model the SRNN-based online regression problem as an online learning problem, where we assume each time step as a separate loss function assigned by an adversary. We characterize the Lipschitz properties of these loss functions with respect to the network weights and derived sufficient conditions for our model to have bounded derivatives. Then, by using these results, we introduce an online gradient descent algorithm that is guaranteed to converge to the locally optimum parameters in a strong deterministic sense, i.e., without any stochastic assumptions. 

Through an extensive set of experiments, we verify our theoretical results and demonstrate significant performance improvements of our algorithm with respect to LSTMs trained with the state-of-the-art methods. In particular, we show that when SRNNs are trained with our algorithm, they provide very similar performance with the LSTMs trained with  the state-of-the-art training methods: Adam, RMSprop, and SGD. Moreover, we observe that our algorithm achieves the equivalent performance in two to three times shorter run-time due to the smaller number of parameters in SRNNs compared to LSTMs. 

As future work, we consider utilizing adaptive learning rate schemes and momentum methods in our algorithm to improve its performance further.  We also plan to test our algorithm on other classes of problems, such as sequence to sequence learning or generative models, to extend its use in practical applications.

\appendices

\section{} 
In this part, we explain why our algorithm can be used with the cross-entropy loss without any change. Since the cross-entropy loss is mainly used to learn binary target values, we naturally assume the following RNN architecture:
\begin{align}
&\htn = \tanh(\W \htm + \U \xt) \nonumber  \\
&\pht =  \sigma(\C^T \htn) \nonumber \\
&E_t = RE(\pt,\pht) \nonumber.
\end{align}
Here,  $\sigma$ is the sigmoid function, i.e., $\sigma(x) = 1/(1+ e^{-x})$,  $\htn \in [-1,1]^{\nh}$ is the hidden vector, $\xt \in [-1,1]^{\nx}$ is the input vector, and $\pt, \pht \in [0,1]$ are the target and estimated values. Moreover, $RE$ denotes the cross-entropy loss, i.e., $RE(\pt, \pht) = - \pt \log \pht - (1 - \pt) \log ( 1 - \pht)$, and $E_t$ denotes the instantaneous loss at time step $t$.

As in the squared loss, the cross-entropy is convex with respect to output layer weights $\C$. Therefore, we can use the projected online gradient descent -- as in (\ref{eq:updc})-- to ensure the convergence of the output layer learning rule. Moreover, the formulation of the derivative of the cross-entropy function with respect to $\C$ is the same with that of the squared loss, i.e.,
\begin{equation*}
\frac{0.5(\dt-\dht)^2}{\del \C}  = (\dht-\dt) \C \textrm{ and } \frac{\del RE(\pt , \pht)}{\del \C} = (\pht-\pt) \C,
\end{equation*}
where $\dt$ and $\dht$ are the outputs of the regression model in (\ref{model1})-(\ref{model2}). Therefore, the Lipschitz properties derived  in Theorem \ref{th:lips} applies to the cross-entropy loss as well.  Since Theorem \ref{th:conv} uses only the Lipschitz properties, it can be extended for the cross-entropy loss with the same learning rules in (\ref{eq:updth})-(\ref{eq:updmu}). As a result, Algorithm \ref{alg:alg1} can be used for the cross-entropy loss without any change.

\section{Preliminaries for the Proofs}
\label{sec:pre}
In the proofs, we use $\norminf{\cdot}$ for the $\ell_\infty$ norm. We denote the derivative of $\tanh$ as $\tanh'$, where $\tanh'(x)= 1- \tanh(x)^2$. We denote the elementary row scaling operation with $\odot$, i.e., $\textbf{x} \odot \textbf{W}= \textrm{diag}(\textbf{x}) \W$. Here, $\textrm{diag}(\textbf{x}) \in \mathbbm{R}^{n \times n}$ is the elementary scaling matrix whose diagonal elements are the components of $\textbf{x} \in \mathbbm{R}^{n}$.  

For the following analysis, we reformulate our network model as
\begin{align}
&\dht =  \C^T \htn.\label{model2} \\
&\htp = \tanh(\W \htn + \U \xtp ). \label{model1} 
\end{align}
Moreover, we write the hidden state update in (\ref{model1}) with the vectorized weight matrices as
\begin{equation}
\htp = \tanh(\Htn \thk + \Xtp \muk) \label{model3}
\end{equation}
where $\Htn= \I \otimes \htn^T$, $\Xt= \I \otimes \xt^T$, and $\otimes$ is the Kronecker product.

\section{Auxiliary Propositions}
\label{sec:aux}
\begin{prop}
\label{prop:aux1}
For any $\textbf{x}, \textbf{y} \in \mathbbm{R}^n$, $\textbf{W} \in \mathbbm{R}^{n \times m}$, where $n,m \in \mathbbm{N}$, the following statements hold:
\begin{align}
&1 ) \: \norm{\textbf{x} \odot \W} \leq \norminf{\textbf{x}} \norm{\W} \hspace{42mm} \label{lem3:1} \\
&2 ) \: \norminf{\tanh'(\textbf{x})- \tanh'(\textbf{y}) } \leq 2 \norm{\textbf{x}-\textbf{y}} \label{lem3:2} \\
&3 ) \: \norm{\I \otimes \textbf{x}^T} = \norm{\textbf{x}}. \label{lem3:3}
\end{align}
\end{prop}
\begin{proof}[Proof of Proposition \ref{prop:aux1}]
\begin{enumerate}
\item Since $\textbf{x} \odot \textbf{W}= \textrm{diag}(\textbf{x}) \W$, we have $ \norm{\textbf{x} \odot \W} \leq \norm{\textrm{diag}(\textbf{x})} \norm{\W}$, where we use the Cauchy-Schwarz inequality for bounding. Since by definition $\norm{\textrm{diag}(\textbf{x})}= \norminf{\textbf{x}}$, $\norm{\textbf{x} \odot \W}  \leq \norm{\textrm{diag}(\textbf{x})} \norm{\W} = \norminf{\textbf{x}} \norm{\W}$.
\item Recall that $\tanh'(x)= 1- \tanh(x)^2$. Since $\tanh(x) \in [-1,1]$, $\tanh$ is $1$-Lipschitz $2$-smooth. Then, by using  $\norminf{\textbf{x}} \leq \norm{\textbf{x}}$ for any $\textbf{x} \in \mathbbm{R}^n,$ we have $\norminf{\tanh'(\textbf{x})- \tanh'(\textbf{y}) } \leq \norm{\tanh'(\textbf{x})- \tanh'(\textbf{y}) }$. Since $\tanh$ is $2$-smooth, we have  $\norm{\tanh'(\textbf{x})- \tanh'(\textbf{y}) }  \leq 2 \norm{\textbf{x}-\textbf{y}}$.
\item See \cite[Theorem 8]{Kron}.
\end{enumerate}
\end{proof}

\begin{prop}
\label{prop:aux2}
Let $\W$ and $\U$ be the hidden-layer weight matrices in (\ref{model1}) satisfying $\norm{\W} \leq \lambda$, and $\norm{\U} \leq \lambda$ for some $\lambda \in \mathbbm{R}$. By using the formulation given in (\ref{model3}), the Lipschitz and smoothness properties of the single SRNN iteration can be written as:
\begin{spreadlines}{0.8em}
\begin{align}
&1 ) \: \bnorm{ \frac{\del \tanh(\Htn \thk + \Xtp \muk)}{\del \htn } } \! \leq  \! \lambda, \hspace{33mm} \label{eq:h} \\ 
&2 ) \: \bnorm{ \frac{\del^2 \! \tanh(\Htn \thk + \Xtp \muk)}{\del \htn^2} } \! \leq \! 2 \lambda^2,  \label{eq:hh}  \\
&3 ) \: \bnorm{ \frac{\del^2 \! \tanh(\Htn \thk + \Xtp \muk)}{\del \htn \del \thk} } \! \leq \! 2 \lambda \sqrt{\nh}, \label{eq:ht}  \\
&4 ) \: \bnorm{ \frac{\del^2 \! \tanh(\Htn \thk + \Xtp \muk)}{\del \htn \del \muk} } \! \leq \! 2 \lambda \sqrt{\nx}, \label{eq:hm} \\
&5 ) \: \bnorm{ \frac{\del^2 \! \tanh(\Htn \thk + \Xtp \muk)}{\del \thk \del \muk} } \! \leq \! 2 \sqrt{\nx} \sqrt{\nh}. \label{eq:tm} \\
&6 ) \: \bnorm{ \frac{\del \tanh(\Htn \thk \! + \! \Xtp \muk)}{\del \thk}  } \! \leq \! \sqrt{\nh},  \label{eq:t}  \\
&7 ) \: \bnorm{ \frac{\del \tanh(\Htn \thk \! + \! \Xtp \muk)}{\del \muk}  } \! \leq \! \sqrt{\nx},  \label{eq:m}  \\
&8 ) \: \bnorm{ \frac{\del^2  \!\tanh(\Htn \thk \! + \! \Xtp \muk)}{\del \thk^2} } \! \leq \! 2 \nh, \label{eq:tt}  \\
&9 ) \: \bnorm{ \frac{\del^2 \! \tanh(\Htn \thk \! + \! \Xtp \muk)}{\del \muk^2} } \! \leq \! 2 \nx. \label{eq:mm} 
\end{align}
\end{spreadlines}
\end{prop}

\begin{proof}[Proof of Proposition \ref{prop:aux2}] In the following, we prove each statement separately:
\begin{enumerate}[wide=0pt]
\item We note that (\ref{model1}) and (\ref{model3}) are equivalent. By using (\ref{lem3:1}) and $\tanh'(x) \leq 1$ on (\ref{model1}), we write
\begin{align*} 
\bnorm{ \frac{\partial \tanh(\W \htn + \U \xtp) }{\partial \htn } } & =  \norm{ \tanh'(\W \htn + \U \xtp) \odot \W } \\
& \leq  \norminf{\tanh'(\W \htn + \U \xtp)} \norm{\W }  \\
& \leq \lambda.
\end{align*} 
\item  By using (\ref{lem3:1}) and (\ref{lem3:2}), we write
\begin{spreadlines}{0.8em}
\begin{align*}
&\bnorm{ \frac{\partial \tanh(\W \htn + \U \xtp)}{\partial \htn} - \frac{\partial \tanh(\W' \htn + \U \xtp)}{\partial \htn}}\\ 
&= \norm{ \tanh'(\W \htn + \U \xtp) \odot \W - \tanh'(\W \htn' + \U \xtp) \odot \W} \\
&\leq  \norminf{\tanh'(\W \htn + \U \xtp) - \tanh'(\W \htn' + \U \xtp) } \norm{\W}  \\
&\leq 2 \norm{\W}  \norm{\htn - \htn'} \norm{\W}  \leq 2 \lambda^2 \norm{\htn - \htn'}.
\end{align*} 
\end{spreadlines}
\item We note that since $\tanh$ is twice differentiable, the order of partial derivatives is not important. Then, by using  (\ref{lem3:1}), (\ref{lem3:2}), (\ref{lem3:3}), and $\norm{\htn} \leq \sqrt{\nh}$, we write
\begin{spreadlines}{0.8em}
\begin{align}
&\bnorm{ \frac{\partial \tanh(\Htn \thk + \Xtp \muk)}{\partial \thk} - \frac{\partial \tanh(\Htn' \thk + \Xtp \muk)}{\partial \thk}  } \nonumber \\
& {\footnotesize = \norm{ \tanh'(\Htn \thk + \Xtp \muk) \odot \Htn \! - \! \tanh'(\Htn' \thk + \Xtp \muk) \odot \Htn'} } \label{lem1:m1} \\
&\leq  \norminf{ \tanh'(\Htn \thk + \Xtp \muk)  - \tanh'(\Htn' \thk + \Xtp \muk) } \norm{\Htn}   \nonumber \\
& \quad + \norminf{ \tanh'(\Htn' \thk + \Xtp \muk) } \norm{\Htn- \Htn'} \label{lem1:0} \\
& \leq  \norminf{\tanh'(\W \htn + \U \xtp)  -  \tanh'(\W \htn' + \U \xtp) } \norm{\Htn}  \nonumber \\
&\quad  + \norm{\htn-\htn'} \\
& \leq \big(2 \norm{\W} \sqrt{\nh} \!  + \! 1 \big) \norm{\htn \! - \! \htn'} \leq (2 \lambda \sqrt{\nh} \! + \! 1 ) \norm{\htn - \htn'},  \nonumber
\end{align}
\end{spreadlines}
where we add $\pm \tanh'(\Htn' \thk + \Xtp \muk) \odot \Htn$ inside of the norm in (\ref{lem1:m1}), and use the triangle inequality for (\ref{lem1:0}). Here, we omit $+1$ term for mathematical convenience in the following derivations.
\item This can be obtained by repeating the steps in the proof of (\ref{eq:ht}) for $\muk$ and $\htn$.
\item By using  (\ref{lem3:1}), (\ref{lem3:2}), (\ref{lem3:3}), $\norm{\htn} \leq \sqrt{\nh}$ and $\norm{\xt} \leq \sqrt{\nx}$, we write
\begin{spreadlines}{0.8em}
\begin{align*}
&\bnorm{ \frac{\partial \tanh(\Htn \thk + \Xtp \muk)}{\partial \thk} - \frac{\partial \tanh(\Htn \thk + \Xtp \muk')}{\partial \thk}  } \\
& {\footnotesize = \norm{\tanh'(\Htn \thk + \Xtp \muk) \odot \Htn - \tanh'(\Htn \thk + \Xtp \muk') \odot \Htn} }\\
& \leq \norminf{\tanh'(\Htn \thk + \Xtp \muk) - \tanh'(\Htn \thk + \Xtp \muk') }  \norm{\Htn} \\
& \leq 2 \norm{\Xtp} \norm{\muk-\muk'} \norm{\Htn} \\
& \leq 2 \sqrt{\nh} \sqrt{\nx} \norm{\muk-\muk'}.
\end{align*} 
\end{spreadlines}
\item By using  (\ref{lem3:1}), $\tanh'(x) \leq 1$, and $\norm{\htn} \leq \sqrt{\nh}$,
\begin{spreadlines}{0.8em}
\begin{align*}
\bnorm{ \frac{\partial \tanh(\Htn \thk + \Xtp \muk)}{\partial \thk}  } & =  \norm{\tanh'(\Htn \thk + \Xtp \muk)\odot \Htn}  \\
& \leq  \norminf{\tanh'(\Htn \thk + \Xtp \muk)} \norm{ \Htn } \\
& \leq \sqrt{\nh}.
\end{align*} 
\end{spreadlines}
\item This can be obtained by repeating the steps in the proof of (\ref{eq:t}) for $\muk$.
\item By using  (\ref{lem3:1}), (\ref{lem3:2}), (\ref{lem3:3}), and $\norm{\htn} \leq \sqrt{\nh}$, we write
\begin{spreadlines}{0.8em}
\begin{align*}
& \bnorm{ \frac{\partial \tanh(\Htn \thk + \Xtp \muk)}{\partial \thk} - \frac{\partial \tanh(\Htn \thk' + \Xtp \muk)}{\partial \thk}  } \\
&= {\footnotesize \norm{ \tanh'(\Htn \thk + \Xtp \muk) \odot \Htn - \tanh'(\Htn \thk' + \Xtp \muk) \odot \Htn} } \\
&\leq \norminf{\tanh'(\Htn \thk + \Xtp \muk)- \tanh'(\Htn \thk' + \Xtp \muk)} \norm{\Htn}  \\
& \leq 2 \norm{\Htn}  \norm{\thk-\thk'} \norm{\Htn} \\
&\leq 2 \nh \norm{\thk-\thk'}.
\end{align*} 
\end{spreadlines}
\item This can be obtained by repeating the steps in the proof of (\ref{eq:tt}) for $\muk$.
\end{enumerate}

\end{proof}

\section{Proof of Lemma 1}
\label{sec:l1}
\begin{proof}[Proof of Lemma 1]
Before the proof, let $\htn(\thk',\muk)$ be the state vector obtained at time $t$ by running the model in (\ref{model1}) with the matrices $\W'$, $\U$ , input sequence $\{\textbf{x}_1,\textbf{x}_2,\cdots,\textbf{x}_{t}\}$, and the initial condition $\htb(\thk',\muk)=\htb(\thk, \muk)= \htb(\thk', \muk')$. Then,
\begin{align}
&\norm{\htn(\thk, \muk) - \htn(\thk', \muk')} \nonumber   \\
&\hspace{8mm} \quad \leq \norm{\htn(\thk, \muk) - \htn(\thk', \muk)}  + \norm{\htn(\thk', \muk) - \htn(\thk', \muk')}, \label{lem2:2} 
\end{align}
where we add  $\pm \htn(\thk',\muk)$ inside of the norm and use the triangle inequality. 

We will bound the terms in (\ref{lem2:2}) separately. We begin with the first term. Since $\htn(\thk, \muk)$ and $\htn(\thk', \muk)$ include the same $\muk$, in the following (between (\ref{lem2:3})-(\ref{lem2:7})), we abbreviate them as $\htn(\thk)$ and $\htn(\thk')$: 
\begin{spreadlines}{0.8em}
\begin{align}
& \norm{\htn(\thk) - \htn(\thk')} =  \label{lem2:3}\\
&= \norm{\tanh(\W \htm(\thk) + \U \xt) - \tanh(\W' \htm(\thk') + \U \xt)} \label{lem2:4}  \\
&\leq \norm{\tanh(\W \htm(\thk) + \U \xt) - \tanh(\W \htm(\thk') + \U \xt)} \nonumber \\
& \quad  +\norm{\tanh(\W \htm(\thk') + \U \xt) - \tanh(\W' \htm(\thk') + \U \xt)}  \label{lem2:5} \\
&\leq \lambda \norm{ \htm(\thk)- \htm(\thk')} + \sqrt{\nh} \norm{\thk- \thk'} \label{lem2:6} \\
&\leq  \sum_{i=0}^{t-1} \Big( \lambda^i  \sqrt{\nh} \norm{\thk- \thk'} \Big) \label{lem2:7} 
\end{align}
\end{spreadlines}
Here, to obtain (\ref{lem2:5}), we add $\pm \tanh(\W \htm(\thk') + \U \xt)$ inside of the norm in (\ref{lem2:4}), and use the triangle inequality. Then, we use (\ref{eq:h}) and (\ref{eq:t}) to get (\ref{lem2:6}). Until we reach (\ref{lem2:7}), we repeatedly apply the same bounding technique to bound the norm of the differences between the state vectors.

Now, we bound the second term in (\ref{lem2:2}). Since $\htn(\thk', \muk)$ and $\htn(\thk', \muk')$ include the same $\thk'$, in the following (between (\ref{lem2:9})-(\ref{lem2:13})), we abbreviate them as $\htn(\muk)$ and $\htn(\muk')$: 
\begin{spreadlines}{0.8em}
\begin{align}
& \norm{\htn(\muk) - \htn(\muk')} = \label{lem2:9} \\
&= \norm{\tanh(\W' \htm(\muk) \! + \! \U \xt) \!- \! \tanh(\W' \htm(\muk') \! + \! \U' \xt)} \label{lem2:10}  \\
& \leq \norm{\tanh(\W' \htm(\muk) \! + \! \U \xt) \!-\! \tanh(\W' \htm(\muk') \! + \! \U \xt)} \nonumber \\
& \quad  +\norm{\tanh(\W' \htm(\muk') \! + \! \U \xt) \! - \! \tanh(\W' \htm(\muk') \!+ \!\U' \xt)}  \label{lem2:11}\\
& \leq  \lambda \norm{ \htm(\muk)- \htm(\muk')} + \sqrt{\nx} \norm{\muk- \muk'} \label{lem2:12} \\
& \leq \sum_{i=0}^{t-1} \Big( \lambda^i \sqrt{\nx} \norm{\muk- \muk'}  \Big), \label{lem2:13} 
\end{align} 
\end{spreadlines}
where for (\ref{lem2:11}), we add $\pm \tanh(\W' \htm(\muk') + \U \xt)$ and use the triangle inequality. We, then, use (\ref{eq:h}) and (\ref{eq:m}) to get (\ref{lem2:12}).  Until we reach (\ref{lem2:13}), we repeatedly apply the same technique to bound the norm of the differences between state vectors. In the end, we use (\ref{lem2:7}) and (\ref{lem2:13}) to bound (\ref{lem2:2}), which yields the statement in the lemma.
\end{proof}

\section{Proof of Theorem 1}
\label{sec:th1}
\begin{proof}[Proof of Theorem 1]
Recall that 
$$\slossgen = \frac{1}{w} \sum_{i=0}^{w-1} \ell_{t-i}(\thk,\muk).$$ 
Then, if we bound the derivative of $\ell_{t}(\thk,\muk)$ for an arbitrary $t \in [T]$, the resulting bound will be  valid for $\slossgen$ as well. Therefore, here, we analyze the Lipschitz properties of $\ell_{t}(\thk,\muk)$ for an arbitrary $t \in [T]$ and extend the result to $\slossgen$.  In the following, we prove each statement of the theorem separately.
\begin{enumerate}[wide=0pt]
\item Let us use $\htn$ and $\htn'$ for the state vectors obtained by running the model in (\ref{model1})  from the initial step up to current time step $t$ with the same initial condition, same input layer matrix $\muk$, common input sequence $\{\textbf{x}_1,\cdots, \textbf{x}_{t}\}$ but different $\thk$ and $\thk'$, respectively.  Let us also say $\losst= 0.5 (\dt-\dht')^2$, where $\dht'$ is the prediction of the second model producing $\htn'$.  Then,
\begin{spreadlines}{0.8em} 
\begin{align}
&\bnorm{\frac{\del \loss}{\del \thk}- \frac{\del \losst}{\del \thk}} \nonumber \\
& = \!\bnorm{(\dt \! - \! \dht') \C^T \Big(\! \sum_{\tau=1}^t \! \phtntp \phtthp \!\Big)\!- \!(\dt \!-\!\dht) \C^T \Big(\! \sum_{\tau=1}^t \! \phtnt \phtth \! \Big)  } \label{th12:2} \\
& \leq  2 \sqrt{\nh} \sum_{\tau=1}^t \bnorm{ \phtnt \phtth- \phtntp \phtthp } \label{th12:3} \\
& \footnotesize{ \leq \! 2 \sqrt{\nh}  \sum_{\tau=1}^t \! \Big( \bnorm{ \phtnt \phtth \! - \! \phtnt \phtthp } \! + \! \bnorm{ \phtthp \phtnt \! - \! \phtthp \phtntp }  \Big) } \label{th12:35} \\
& \leq \! 2 \sqrt{\nh}  \sum_{\tau=1}^t \! \Big( \bnorm{ \phtnt } \bnorm{ \phtth \! - \! \phtthp } \! + \! \bnorm{ \phtthp } \bnorm{ \phtnt \! - \! \phtntp }  \Big) \label{th12:4} \\
& \leq 2 \sqrt{\nh} \sum_{\tau=1}^t \lambda^{t-\tau} \bnorm{ \phtth \! - \! \phtthp } \! + \! 2 \nh \sum_{\tau=1}^t  \bnorm{ \phtnt  \! - \! \phtntp }. \label{th12:5}
\end{align} 
\end{spreadlines}
Here, we use the bounds of $\dt$, $\dht$ and $\C$ for (\ref{th12:2}).\footnote{Note that since $\dt , \dht \in [-\sqrt{\nh},\sqrt{\nh}]$ and $\norm{\Ct} \leq 1$, the $\ell_2$ norm of $(\dt-\dht) \C$ is bounded by $2 \sqrt{\nh}$, i.e., $\norm{(\dt-\dht) \C} \leq 2 \sqrt{\nh}$.}  To get (\ref{th12:35}), we add $\pm \phtnt \phtthp$ inside the norm of (\ref{th12:3}), and use the triangle inequality. To get (\ref{th12:5}), we use (\ref{eq:h}) and (\ref{eq:t}). 

In the following, we will bound the terms in (\ref{th12:5}) separately. We begin with the first term. Note that $\htt= \tanh( \W \httm + \U \textbf{x}_{\tau})$, and $\htt'= \tanh( \W' \httm' + \U \textbf{x}_{\tau})$. Then,
\begin{spreadlines}{0.8em}
\begin{align}
2 &  \sqrt{\nh}  \sum_{\tau=1}^t   \lambda^{t-\tau} \bnorm{ \phtth \!  - \!  \phtthp } \label{th12:525}  \\
& \leq 2 \sqrt{\nh}    \sum_{\tau=1}^t   \lambda^{t-\tau} \bnorm{ \phtth \!  -  \! \frac{\del \tanh( \W \httm' + \U \textbf{x}_{\tau-1}) }{\del \thk}  } \nonumber \\
&\quad  + 2 \sqrt{\nh}    \sum_{\tau=1}^t  \lambda^{t-\tau} \bnorm{ \frac{\del \tanh( \W \httm' + \U \textbf{x}_{\tau-1}) }{\del \thk} \! - \!  \phtthp } \label{th12:55} \\
& \leq  2 \sqrt{\nh} \sum_{\tau=1}^t   \lambda^{t-\tau} \big( 2 \lambda \sqrt{\nh} \norm{\httm \! - \! \httm'} \! + \! 2 \nh \norm{\thk \! - \! \thk'} \big) \label{th12:6} \\
& \leq 2 \sqrt{\nh} \Big( \! \sum_{\tau=1}^t \lambda^{t-\tau} \! \Big) \Big( \frac{2 \lambda \nh}{1-\lambda} + 2 \nh \! \Big) \norm{\thk \!  - \!  \thk'} \label{th12:65}  \\
& \leq \frac{ 4 \nh \sqrt{\nh}}{1-\lambda} \Big( \frac{ \lambda }{1-\lambda} + 1 \Big) \norm{\thk \!  - \! \thk'}, \label{th12:665}
\end{align} 
\end{spreadlines}
Here, to get (\ref{th12:55}), we add $\pm \frac{\del \tanh( \W \httm' + \U \textbf{x}_{\tau-1}) }{\del \thk}$ inside of the norm in (\ref{th12:525}) and use the triangle inequality. We use (\ref{eq:ht}) and (\ref{eq:tt}) for (\ref{th12:6}). We use Lemma 1 and $\lambda \in [0,1)$ for (\ref{th12:65}).

Now, we bound the second term in (\ref{th12:5}). To bound the term, we first focus on the term inside of the sum, i.e., $\bnorm{ \phtnt  -  \phtntp }$:
\begin{spreadlines}{0.8em}
\begin{align}
&  \bnorm{ \phtnt  \!- \!  \phtntp }  \leq \bnorm{\frac{\del \htm}{\del \htt}} \bnorm{\frac{\del \htn}{\del \htm}-\frac{\del \htn'}{\del \htm'}} \nonumber\\
& \qquad \qquad \qquad \qquad   + \bnorm{\frac{\del \htn'}{\del \htm'}}  \bnorm{ \frac{\del \htm}{\del \htt}-\frac{\del \htm'}{\del \htt}} \label{th12:675}\\
& \leq \! \lambda^{t-\tau-1} \big( 2 \lambda^2 \norm{\htm \! -  \! \htm'}  \! + \! 2 \lambda \sqrt{\nh} \norm{\thk \! - \! \thk'} \big)  \nonumber \\
& \qquad \qquad \qquad \qquad  +  \lambda \bnorm{\frac{\del \htm}{\del \htt}  \! - \! \frac{\del \htm'}{\del \htt}} \label{th12:7} \\
&  \leq  \lambda^{t-\tau-1} \Big( \frac{ 2 \lambda^2 \sqrt{\nh}}{1-\lambda}  \! + \! 2 \lambda \sqrt{\nh}  \Big) \norm{\thk \! - \! \thk'} \! + \!  \lambda \bnorm{\frac{\del \htm}{\del \htt}  \! - \! \frac{\del \htm'}{\del \htt}} \label{th12:75} \\
& \leq (t \! - \! \tau) \lambda^{t-\tau-1} \Big( \frac{ 2 \lambda^2 \sqrt{\nh}}{1-\lambda} \! + \! 2 \lambda \sqrt{\nh} \Big) \norm{\thk-\thk'}, \label{th12:8}
\end{align}
\end{spreadlines}
where we add  $\pm \frac{\del \htn' }{\del \htm'} \frac{\del \htm }{\del \htt}$, and use the triangle inequality for (\ref{th12:675}). We use (\ref{eq:hh}) and (\ref{eq:ht}) for (\ref{th12:7}). We use Lemma 1 and $\lambda \in [0,1)$ for (\ref{th12:75}).  We, then, repeat the same manipulations in (\ref{th12:675})-(\ref{th12:75}) to bound the terms with partial derivatives. 

Then, the second term in (\ref{th12:5}) can be bound as:
\begin{spreadlines}{0.8em}
\begin{align}
& 2 \nh \sum_{\tau=1}^t \bnorm{ \phtnt - \phtntp } \\
&\leq 2 \nh  \sum_{\tau=1}^t (t-\tau) \lambda^{t-\tau-1} \Big( \frac{ 2 \lambda^2 \sqrt{\nh}}{1-\lambda}+ 2 \lambda \sqrt{\nh} \Big) \norm{\thk-\thk'} \\
& = \frac{4 \nh \sqrt{\nh}}{1-\lambda} \Big(  \frac{\lambda^2 }{(1-\lambda)^2}  + \frac{\lambda }{1-\lambda} \Big) \norm{\thk-\thk'}. \label{th12:9}
\end{align}
\end{spreadlines}
Here, to get (\ref{th12:9}), we use  the fact that  
$$\sum_{\tau=1}^t (t-\tau) \lambda^{t-\tau-1} \leq \frac{1}{(1-\lambda)^2} \textrm{ for any } t \in \mathbbm{N}.$$

Then, by using (\ref{th12:665}) and (\ref{th12:9}), we can bound (\ref{th12:5}) as
\begin{spreadlines}{0.8em}
\begin{align}
&\bnorm{\frac{\del \loss}{\del \thk}- \frac{\del \losst}{\del \thk}} \\
& \hspace{22mm} \leq  \frac{4 \nh \sqrt{\nh}}{1-\lambda} \Big(  \frac{\lambda^2 }{(1-\lambda)^2} \! + \! \frac{2 \lambda }{1-\lambda}  +  1 \Big) \norm{\thk \! -  \! \thk'} \\
& \hspace{22mm} = \frac{4 \nh \sqrt{\nh}}{(1-\lambda)^3} \norm{\thk-\thk'}. \label{th12:10}
\end{align}
\end{spreadlines}
By realizing that (\ref{th12:10}) holds for an arbitrary $t$, the statement in the theorem can be obtained.

\item This part can be obtained by adapting the steps in the previous proof for $\muk$, and use the Lipschitz conditions in (\ref{eq:hm}), (\ref{eq:m}), (\ref{eq:mm}) accordingly.

\item We use the same notation as in the proof of the 1st statement. Then,
\begin{align}
& \bnorm{\frac{\del \loss}{\del \muk}- \frac{\del \losst}{\del \muk}} \nonumber \\
&  = \!\bnorm{(\dt \! - \! \dht') \C^T \Big(\! \sum_{\tau=1}^t \! \phtntp \phtmup \!\Big)\!- \!(\dt \!-\!\dht) \C^T \Big(\! \sum_{\tau=1}^t \! \phtnt \phtmu \! \Big)  } \label{th15:1} \\
& \leq  2 \sqrt{\nh} \sum_{\tau=1}^t \bnorm{ \phtnt \phtmu- \phtntp \phtmup } \label{th15:2} \\
& \footnotesize{ \leq \! 2 \sqrt{\nh}  \sum_{\tau=1}^t \! \Big( \bnorm{  \phtnt \phtmu \! - \!  \phtnt \phtmup } \! + \! \bnorm{ \phtmup  \phtnt \! - \! \phtmup  \phtntp }  \Big) } \label{th15:25} \\
& \leq \! 2 \sqrt{\nh}  \sum_{\tau=1}^t \! \Big( \bnorm{ \phtnt } \bnorm{ \phtmu \! - \! \phtmup } \! + \! \bnorm{ \phtmup } \bnorm{ \phtnt \! - \! \phtntp }  \Big) \label{th15:3} \\
& \leq 2 \sqrt{\nh} \sum_{\tau=1}^t \lambda^{t-\tau} \bnorm{ \phtmu \! - \! \phtmup } \! + \! 2 \sqrt{\nh \nx} \sum_{\tau=1}^t \! \bnorm{ \phtnt  \! - \! \phtntp }. \label{th15:4}
\end{align} 
Here, to get (\ref{th15:25}), we add $\pm \phtnt \phtmup$ inside the norm in (\ref{th15:2}), and use the triangle inequality. To get (\ref{th15:4}), we use (\ref{eq:h}) and (\ref{eq:m}). 

We bound the terms in (\ref{th15:4}) separately. We begin with the first term. Note that $\htt= \tanh( \W \httm + \U \textbf{x}_{\tau})$, and $\htt'= \tanh( \W' \httm' + \U \textbf{x}_{\tau})$. Then,
\begin{align}
& 2 \sqrt{\nh} \sum_{\tau=1}^t  \lambda^{t-\tau} \bnorm{ \phtmu \!  - \!  \phtmup } \label{th15:455}  \\
& \leq 2 \sqrt{\nh} \sum_{\tau=1}^t  \lambda^{t-\tau} \bnorm{ \phtmu \!  - \! \frac{\del \tanh( \W \httm' + \U \textbf{x}_{\tau-1}) }{\del \thk}} \nonumber \\
& \quad 2 \sqrt{\nh} \sum_{\tau=1}^t  \lambda^{t-\tau} \bnorm{ \frac{\del \tanh( \W \httm' + \U \textbf{x}_{\tau-1}) }{\del \thk} \!  - \!  \phtmup }  \label{th15:45} \\
&\leq  2 \sqrt{\nh} \sum_{\tau=1}^t  \lambda^{t-\tau} \big( 2 \lambda \sqrt{\nx} \norm{\httm \! - \! \httm'} \! + \! 2 \sqrt{\nx \nh} \norm{\thk \! - \! \thk'} \big) \label{th15:5} \\
& \leq 2 \sqrt{\nh}  \Big( \! \sum_{\tau=1}^t \lambda^{t-\tau} \! \Big) \Big( \frac{2 \lambda \sqrt{\nx \nh}}{1-\lambda} + 2 \sqrt{\nx \nh} \! \Big) \norm{\thk \!  - \!  \thk'} \label{th15:6}  \\
&  \leq \frac{4 \nh \sqrt{\nx}}{1-\lambda} \Big( \frac{ \lambda }{1-\lambda} + 1 \Big) \norm{\thk \!  - \! \thk'}, \label{th15:7}
\end{align} 
where to get (\ref{th15:45}), we add $\pm \frac{\del \tanh( \W \httm' + \U \textbf{x}_{\tau-1}) }{\del \thk}$ inside the norm in (\ref{th15:455}), and use the triangle inequality,. We use (\ref{eq:hm}) and (\ref{eq:tm}) for (\ref{th15:5}). We use Lemma 1 and $\lambda \in [0,1)$ for (\ref{th15:6}). 

Now, we bound the second term in (\ref{th15:4}):
\begin{align}
&  2 \sqrt{\nh \nx} \sum_{\tau=1}^t \bnorm{ \phtnt - \phtntp } \\
& \leq 2 \sqrt{\nh \nx}  \sum_{\tau=1}^t (t-\tau) \lambda^{t-\tau-1} \Big( \frac{ 2 \lambda^2 \sqrt{\nh}}{1-\lambda}+ 2 \lambda \sqrt{\nh} \Big) \norm{\thk-\thk'} \\
& = \frac{4 \nh \sqrt{\nx}}{1-\lambda} \Big(  \frac{\lambda^2 }{(1-\lambda)^2}  + \frac{\lambda }{1-\lambda} \Big) \norm{\thk-\thk'}, \label{th15:8}
\end{align}
where we  use (\ref{th12:8}) to bound the terms $\bnorm{ \phtnt - \phtntp }$. 

Then, by using (\ref{th15:7}) and (\ref{th15:8}), we bound (\ref{th15:4}) as follows:
\begin{align}
& \bnorm{\frac{\del \loss}{\del \muk}- \frac{\del \losst}{\del \muk}} \nonumber  \\
& \qquad \qquad \quad     \leq  \frac{4 \nh \sqrt{\nx}}{1-\lambda} \Big(  \frac{\lambda^2 }{(1-\lambda)^2} \! + \! \frac{2 \lambda }{1-\lambda}  +  1 \Big) \norm{\thk \! -  \! \thk'} \nonumber \\
&\qquad \qquad \quad    = \frac{ 4 \nh \sqrt{\nx}}{(1-\lambda)^3} \norm{\thk-\thk'}. \label{th15:9}
\end{align}
By realizing that (\ref{th15:9}) holds for an arbitrary $t$, the statement in the theorem can be obtained.
\end{enumerate}
\end{proof}

\section{Proof of Theorem 2}
\label{sec:th2} 
\begin{proof}[Proof of Theorem 2]
In the following, we use $\anglep{ \cdot, \cdot}$ to denote the inner product. Due to the space constraints, we  omit  the arguments in the partial derivative terms, i.e., 
\begin{alignat*}{2}
&\shortLTht  \coloneqq \partLTht, &&\shortLprojTht \coloneqq \projLTht, \\
&\shortLMt  \coloneqq  \partLMt, \qquad   &&\shortLprojMt  \coloneqq \projLMt.
\end{alignat*}
In the following, we bound the sums $\sum_{t=1}^T \bnorm{\shortLprojTht}^2$ and  $\sum_{t=1}^T \bnorm{\shortLprojMt}^2$ separately. To bound the first sum, we first fix an arbitrary $\muk \in \Km$ and derive an upper-bound for the difference $L_{t,w}( \thk_{t+1}, \muk) - L_{t,w}( \thk_t, \muk)$, i.e.,
\begin{align}
L_{t,w}( \thk_{t+1}, \muk) &- L_{t,w}( \thk_t, \muk) \nonumber  \\
& \leq  \banglep{ \shortLTht, \thk_{t+1}  -  \thk_t }  + \frac{\bTh}{2} \norm{ \thk_{t+1} -  \thk_t}^2 \label{th2:p1}  \\ 
&\leq - \eta \banglep{ \shortLTht, \shortLprojTht } + \frac{\beta \eta^2}{2} \bnorm{\shortLprojTht}^2 \label{th2:p2}   \\
&\leq - \eta \bnorm{\shortLprojTht}^2 + \frac{\beta \eta^2}{2} \bnorm{\shortLprojTht}^2 \label{th2:p3}\\
&=  - \big( \eta -  \frac{\beta \eta^2}{2} \big) \bnorm{\shortLprojTht}^2  \label{th2:p4}
\end{align}
where we use \cite[Lemma 3.4]{Bubeck15} for (\ref{th2:p1}), Theorem 1 for (\ref{th2:p2}), and \cite[Lemma 3.2]{HazanNonConvex} for (\ref{th2:p3}).

Then, we swap the left and right hand sides of (\ref{th2:p4}), add $\pm L_{t+1,w}( \thk_{t+1}, \muk)$ to the right hand side and upper-bound the terms, i.e.,
\begin{align}
\big( \eta -  \frac{\beta \eta^2}{2} \big) \bnorm{\shortLprojTht}^2  &\leq L_{t,w}( \thk_t, \muk) - L_{t+1,w}( \thk_{t+1}, \muk) \nonumber \\
&+ L_{t+1,w}( \thk_{t+1}, \muk) - L_{t,w}( \thk_{t+1}, \muk) \nonumber \\
&\leq L_{t,w}( \thk_t, \muk) - L_{t+1,w}( \thk_{t+1}, \muk) \nonumber \\
&\quad + \frac{4 \sqrt{\nh}}{w}, \label{th2:p5}
\end{align}
where we use $L_{t,w} \in [- \sqrt{\nh}, \sqrt{\nh}]$ for (\ref{th2:p5}).

Next, we sum both sides of (\ref{th2:p5}) over $T$ and get
\begin{align}
\big( \eta -  \frac{\beta \eta^2}{2} \big)  \sum_{t=1}^T \bnorm{\shortLprojTht}^2 &\leq L_{1,w}( \thk_1, \muk) - L_{T+1,w}( \thk_{T+1} ) \nonumber \\
&\quad +  \frac{4 \sqrt{\nh} T}{w}  \\
&\leq 4 \sqrt{\nh} +  \frac{4 \sqrt{\nh} T}{w} . \label{th2:p6}
\end{align}

By realizing that $\big( \eta -  \frac{\beta \eta^2}{2} \big) > \frac{\eta}{2}$ for $0 < \eta \leq 1/\beta$, we simplify (\ref{th2:p6}) as
\begin{equation}
\sum_{t=1}^T \bnorm{\shortLprojTht}^2 \leq \frac{8 \sqrt{\nh}}{\eta} \frac{T}{w} + \frac{8 \sqrt{\nh}}{\eta}. \label{th2:p7}
\end{equation}

We note that we can bound $\sum_{t=1}^T \bnorm{\shortLprojMt}^2$ by fixing $\thk \in \Kt$ and adapting steps between (\ref{th2:p1})-(\ref{th2:p6}) for $\muk$ accordingly. In the end, we get the same bound, i.e.,
\begin{equation}
\sum_{t=1}^T \bnorm{\shortLprojMt}^2 \leq \frac{8 \sqrt{\nh}}{\eta} \frac{T}{w} + \frac{8 \sqrt{\nh}}{\eta}. \label{th2:p8}
\end{equation}

By summing the inequalities in (\ref{th2:p7}) and (\ref{th2:p8}), we get
\begin{align}
R_w(T) \leq \frac{16 \sqrt{\nh}}{\eta} \frac{T}{w} + \frac{16 \sqrt{\nh}}{\eta}. \label{th2:p9}
\end{align}
\end{proof}

\bibliographystyle{IEEEtran}
\balance
\bibliography{my_references}

\end{document}